\documentclass[twoside,11pt]{article}

\usepackage{blindtext}

\usepackage[abbrvbib, preprint]{jmlr2e}

\usepackage{jmlr2e}
\usepackage{amsmath}
\usepackage{nicematrix}
\usepackage{algorithm}
\usepackage{algpseudocode}
\usepackage{bbm}
\usepackage{float}
\usepackage{caption}

\newcommand{\floor}[1]{\lfloor #1 \rfloor}

\usepackage{lastpage}
\jmlrheading{23}{2024}{1-\pageref{LastPage}}{6/24; Revised x/24}{x/24}{21-0000}{Biagini, Gonon and Walter}

\ShortHeadings{Universal randomised signatures for generative time series modelling}{Biagini, Gonon and Walter}
\firstpageno{1}

\begin{document}

\title{Universal randomised signatures for generative time series modelling}

\author{\name Francesca Biagini \email biagini@math.lmu.de \\
       \addr Department of Mathematics\\
       University Munich\\
       Munich, Germany
       \AND
       \name Lukas Gonon \email l.gonon@imperial.ac.uk \\
       \addr Department of Mathematics\\
       Imperial College London\\
       London, UK
       \AND
       \name Niklas Walter \email walter@math.lmu.de \\
       \addr Department of Mathematics\\
       University Munich\\
       Munich, Germany}

\editor{My editor}

\maketitle

\begin{abstract}
Randomised signature has been proposed as a flexible and easily implementable alternative to the well-established path signature. In this article, we employ randomised signature to introduce a generative model for financial time series data in the spirit of reservoir computing.
Specifically, we propose a novel Wasserstein-type distance based on discrete-time randomised signatures. This metric on the space of probability
measures captures the distance between (conditional) distributions. Its use is justified by our novel universal approximation results for randomised signatures on the space of continuous functions taking the underlying path as an input.
We then use our metric as the loss function in a non-adversarial generator model for synthetic time series data based on a reservoir neural stochastic differential equation. We compare the results of our model to benchmarks from the existing literature. 
\end{abstract}


\section{Introduction}\label{Section1Intro}

In this work, we derive universal approximation results for a reservoir based on the randomised path signature. Based on this, we introduce a novel metric on the space of probability measures and investigate its application in the generation of synthetic financial time series using a non-adversarial training technique.
As described in \cite{Assefa2020, Wiese2019, Buehler2020, Koshiyama2021} the need for synthetic financial (time series) data has grown in recent years for various reasons. First of all, the lack of historical data representing certain events in the financial market might for example lead to incorrectly estimated risk measures such as value-at-risk or expected shortfall. Therefore, it is crucial to generate realistic data for market regimes which rarely or never occurred in the past. Related to that is the growing need for training data necessary to train complex machine learning models. It might be the case that there is not enough training data from historical datasets or that an institution is not able to use some of its data for training purposes. Lastly, synthetic data can ease data sharing for both institutions and academia, since it is not restricted by possible regulations. 

Various techniques have been studied for generating synthetic financial data. One of the pioneering papers is \cite{Kondratyev2019} using restricted Boltzmann machines, firstly introduced in \cite{Smolensky1986}, to sample synthetic foreign exchange rates. In \cite{Buehler2020,Buehler2020a} the authors apply (conditional) variational autoencoders on both the pure returns and signature transformed returns to generate return time series for S\&P 500 data. In \cite{Wiese2021} an autoencoder and normalising flows are used to simulate option markets. This approach also incorporates a necessary no-arbitrage condition based on discrete local volatilities. In \cite{Hamdouche2023} a novel approach for time series generation is proposed using Schr\"odinger bridge. 

In addition to these approaches, the majority of work in the field considers generative adversarial networks (GANs). Firstly proposed in \cite{Goodfellow2014}, these networks are composed of a generator network generating fake data and a discriminator network trying to distinguish between real and fake data, which leads to a min-max game. One of the first papers proposing GANs to generate financial data is \cite{Takahashi2019}, where the generator and discriminator are a multi-layer perceptron, convolutional networks or a combination of both. In \cite{Wiese2020} both the generator and discriminator model consists of temporal convolutional networks addressing long-range dependencies within the time series. In \cite{Fu2022} this dependency is captured by using both convolutional networks with attention and transformers for the GAN architecture. In contrast, \cite{Esteban2017} implement the generator and discriminator pair using LSTMs to generate medical data. The model proposed in \cite{Yoon2019} expands the traditional GAN framework with a trainable embedding space so that it can adhere to the underlying data dynamics during training. Another natural extension are conditional GANs, which do not only take noise as input but also some additional information the generated samples should depend on. They are used in \cite{Koshiyama2021} to produce data used to learn trading strategies and in \cite{Coletta2021} to simulate limit order book data. For applications in a non-financial context we refer to \cite{Mogren2016, Donahue2019, Xu2020} and their corresponding references, where GAN settings are used to produce for example music samples or videos.

In recent years, several works have proposed to substitute the discriminator of the traditional GAN with a Wasserstein-type metric, which leads to more stable training procedures. This so called Wasserstein-GAN was applied in \cite{Wiese2019} to simulate option markets. Although this model can be trained by gradient descent methods, the challenge of solving a min-max problem remains. The authors in \cite{Ni2021} proposed a Wasserstein-type distance on the signature space of the input paths. In this way, the training is transformed into a supervised learning problem. This setting is extended into a conditional setting similar to the CGANs in \cite{Liao2023} and \cite{Lozano2023}. Both papers considered Neural SDEs as the generator of the model. Finally, the authors in \cite{Issa2023} propose a score-based generation technique for financial data using a kernel written on the path signature. Similarly, in \cite{Lu2024}, financial data is generated using a signature kernel and a training based on the maximum mean discrepancy. We refer to \cite{Lyons1998, Hambly2010} for an in-depth study of the path signature and to \cite{Cass2024} for an overview of its application in machine learning.

In this paper, we also replace the GAN discriminator, but instead of using the signature we employ randomised signature as introduced in \cite{Cuchiero2020, Cuchiero2021} and studied in \cite{Compagnoni2023,Schäfl2023,Akyildirim2024,cuchiero2024}. Its connection to signature kernels was examined in \cite{2023NeuralSignatureKernels}. The related concept of path development was studied in \cite{Cass2024a,Lou2022}. The advantage of randomised signature is that it inherits the expressiveness and inductive bias of the signature while being finite dimensional. Therefore, no truncation is necessary as for the signature. Moreover, the computational complexity can be reduced, which makes the randomised signature more tractable for high-dimensional input paths. In particular, we study its properties as a reservoir and derive universal approximation results for continuous function on the input path space. Based on this, we introduce a Wasserstein-distance on the reservoir space induced by the randomised signature. Lastly, we show how this new metric can be used as a discriminator for the non-adversarial training of a generator with a similar structure as the randomised signature, inspired by Neural SDEs \cite{Liu2019a, Kidger2021, Cohen2023}. Here we consider generators with random feature neural networks as coefficient functions. These neural networks have a single hidden layer and the property that the parameters of this layer are randomly initialised and then fixed for the entire training procedure. Hence, only the parameters of the output layer need to be trained. Random feature neural networks were originally introduced in \cite{Huang2006} and \cite{Rahimi2007}.

To justify our approach we prove novel universal approximation results for randomised signature. Our proof builds on a universal approximation result for random feature neural networks from \cite{Hart2020}. Random feature neural networks have been shown to possess universal approximation properties similar to classical neural networks, see \cite{Hart2020, Gonon2023, Gonon2023b, Neufeld2023}. For quantitative error bounds relating to their generalisation properties we refer to \cite{Rudi2017, Carratino2018, Mei2022, Gonon2023} and their corresponding references. 

Based on universal approximation properties of random feature neural networks we are able to derive universal approximation results for linear readouts on the randomised signature. In other words, we study the expressiveness of the randomised signature as a reservoir. More general approximation results for different classes of reservoir computing systems have been studied in \cite{Sontag1979book,Sontag1979} for inputs on finite numbers of time points and in \cite{Sandberg1991, Sandberg1991b, Matthews1992} for semi-infinite inputs. Moreover, in \cite{Grigoryeva2018} the authors derived results for both uniformly bounded deterministic and almost surely bounded stochastic inputs. The latter was extended to more general stochastic inputs in \cite{Gonon2020} and to randomly sampled parameters in \cite{Gonon2023b}. In \cite{Akyildirim2022} a universal approximation result for a randomised signature is derived considering real analytic activation function. 

By introducing both an unconditional and a conditional generative model, we address two alternative important modelling tasks. While the unconditional model aims to learn the unknown distribution of a stochastic process, the conditional method takes given information about a path's history into account and generates possible future trajectories by learning the underlying unknown conditional distribution. We study the models' performances on both synthetic and real world financial datasets. In particular, we consider samples of a Brownian motion with possible drift, an autoregressive process and log-return data from the S\&P 500 index and the FOREX EUR/USD exchange rate obtaining more realistic results than comparable benchmark methods. We assess the quality by providing results of test metrics applied to out-of-sample data. It is worth noting that we implemented the generator in a way such that the variance of the randomly generated weights of the random feature neural networks is trainable. This led to both a higher accuracy and a speed-up of the training procedure, while at the same time much fewer parameters need to optimised than in fully trainable models. Our source code for all experiments is available online at \url{https://github.com/niklaswalter/Randomised-Signature-TimeSeries-Generation}.

The present paper is organised as follows. In Section \ref{Section2}, we introduce a randomised signature for a deterministic datastream attaining values in a compact subset of some higher dimensional space. Moreover, we derive universal approximation results for linear readouts on the randomised signature. In Section \ref{Section3}, we propose a new metric on the nispace of probability measures based on the approximation results in Section \ref{Section2}. Based on this, we introduce a novel GAN framework to generate synthetic financial time series. In Section \ref{section4}, we extend these results to a conditional setting. 

\nocite{*}


\section{Randomised signature of a datastream}\label{Section2}
For some finite time horizon $T\in\mathbb{N}$ let $x:\{1,\dots,T\}\to \mathbb{R}^d$ be a datastream (e.g. a timeseries) of some dimension $d\in\mathbb{N}$. Consider a filtered probability space $(\Omega,\mathcal{F},\mathbb{F},\mathbb{P})$ with the filtration $\mathbb{F}$ satisfying the usual conditions, which we fix for the remainder of this paper. Given $N\in\mathbb{N}$, random matrices $A_1,A_2^1,\dots,A_2^d\in\mathbb{R}^{N\times N}$, random biases $\xi_1,\xi_2^1,\dots,\xi_2^d \in\mathbb{R}^{N}$ defined on $(\Omega,\mathcal{F},\mathbb{P})$ and a fixed activation function $\sigma:\mathbb{R} \to \mathbb{R}$, consider the following controlled system
\begin{equation}\label{eq:reservoirsystem}
        \text{RS}_t(x) = \text{RS}_{t-1}(x) + \sigma(A_1 \text{RS}_{t-1}(x) + \xi_1) + \sum_{i=1}^d \sigma(A_2^i \text{RS}_{t-1}(x) + \xi_2^i) x_t^i, \; t = 1,\dots,T,
\end{equation}
with $\text{RS}_0(x) = 0 \in\mathbb{R}^N$ and $x^i$ denoting th $i$-th component of $x$. With a slight abuse of notation, in the sequel we also denote by $\sigma$ the function defined on $\mathbb{R}^N$ as 
$$
    \sigma\left(\left[\begin{array}{c}
        z_1 \\
        \vdots \\
        z_N
    \end{array}\right]\right) := \left[\begin{array}{c}
        \sigma(z_1) \\
        \vdots \\
        \sigma(z_N)
    \end{array}\right],\, z\in\mathbb{R}^N.
$$

Prominent examples for activation functions are the ReLU, Tanh or Sigmoid function. We later specify our assumption on $\sigma$ in more detail.

\begin{definition}[Randomised signature]\label{def:randomsig}
    Let $x:\{1,\dots,T\}\to\mathbb{R}^d$ be a datastream. For any $t=1,\dots,T$ we call the response $\text{RS}_t(x)$ to the equation \eqref{eq:reservoirsystem} the randomised signature of $x$ at time $t$. We refer to $N\in\mathbb{N}$ as the corresponding dimension of $\text{RS}=(\text{RS}_t)_{t=1,\dots,T}$.
\end{definition}

The coefficient functions of the system $\eqref{eq:reservoirsystem}$ are naturally linked to random feature neural networks, defined in the following.

\begin{definition}[Random feature neural network]\label{def:randomfeatureDNNs}
    Let $N\in\mathbb{N}$. Consider a $\mathbb{R}^{N\times N}$-valued random matrix $A$ and a $\mathbb{R}^N$-valued random vector $\xi$ (often the entries of $A$ and $\xi$ are i.i.d. random variables respectively). For a vector $w\in\mathbb{R}^N$ consider the (random) function
    \begin{equation}\label{eq:randomfeatureDNN}
         \Psi^{A,\xi}_w(x) := w^T \sigma(Ax + \xi),\, x\in\mathbb{R}^N
    \end{equation}
    for some activation function $\sigma:\mathbb{R}^N\to \mathbb{R}^N$. The function $\Psi^{A,\xi}_w$ is called random feature neural network.
\end{definition}

The random variables $A$ and $\xi$ in \eqref{eq:randomfeatureDNN} are called the random hidden weights of the random feature neural network, while $w$ is the vector of the output weights. For the training of the network $\Psi^{A,\xi}_w$ we first consider the hidden weights as sampled and fixed. Afterwards the output vector $w$ is trained in order to achieve a desired approximation accuracy. In particular, the training procedure of random feature neural networks corresponds to a simple linear regression problem, since only the final linear output layer, also called readout, needs to be trained. The following remark shows how equation \eqref{eq:reservoirsystem} can be turned into a neural differential equation as described in the following. 

Applying a trainable linear readout function $\ell \in\left(\mathbb{R}^{N}\right)^*$ to the randomised signature process $\text{RS}(x) = (\text{RS}_t(x))_{t=1,\dots,T}$ at some time point $t=1,...,T$ turns the function in the drift and controlled term into random feature neural networks and equation $\eqref{eq:reservoirsystem}$ into a discrete-time neural differential equation, where the drift and controlled component are random feature neural networks as defined in Definition \ref{def:randomfeatureDNNs}. 

\begin{remark}
    For an output vector $w\in\mathbb{R}^N$ we consider the linear function $\ell_w:\mathbb{R}^N\to\mathbb{R}$ defined as $\ell_w(z):= w^T z$, $z\in\mathbb{R}^N$. For the process defined in \eqref{eq:reservoirsystem} and with the notation $\Delta \text{RS}_t(x) := \text{RS}_t(x) - \text{RS}_{t-1}(x)$ it follows for any $t=1,\dots,T$, that
    $$
        \ell_w(\Delta \text{RS}_t(x)) = w^T \sigma(A_1 \text{RS}_{t-1}(x) + \xi_1) + \sum_{i=1}^d w^T \sigma(A_2^i \text{RS}_{t-1}(x) + \xi_2^i) x_t^i,
    $$
    which defines a discrete-time neural differential equation driven by the input stream $x$. For an in-depth introduction to neural differential equations we refer to \cite{Kidger2021}.
\end{remark}

\subsection{Universal approximation properties}
In this section, we study the properties of the process $\text{RS}(x)$ as a reservoir. In particular, we derive different universal approximation results on linear readouts on the terminal difference $\Delta \text{RS}_T(x)$ approximating continuous functions on the input path $x$. In order to establish universality results, we consider different sampling schemes assuming certain forms for the random matrices and vectors of the system \eqref{eq:reservoirsystem}.

\begin{samplingscheme}\label{samplingscheme1}
    Let $N:= M + Td$ for some $M\in\mathbb{N}$. We consider the following sampling scheme for the hidden weights of $\eqref{eq:reservoirsystem}$. For $i=1,\dots,d$ we assume 
    \begin{equation*}
        A_1 = \left[\begin{array}{cc}
             0_{M\times M} & A^{(1)}_1 \\
             0_{Td\times M} & A^{(2)}_1
        \end{array}\right],\; \xi_1 = \left[\begin{array}{c}
             \xi^{(1)}_1\\
             0_{Td\times 1}
        \end{array}\right],\; A^i_2 = \left[\begin{array}{cc}
             0_{M\times M} & 0_{M\times Td} \\
             0_{Td\times M} & 0_{Td \times Td}
        \end{array}\right],\;  \xi^i_2 = \left[\begin{array}{c}
             0_{M\times 1}\\
             \alpha_i e_i
        \end{array}\right],
    \end{equation*}
    where $e_i$ denotes the $i$-th canonical basis vector of $\mathbb{R}^{Td}$, $\xi^{(1)}_1\in\mathbb{R}^{M}$, $A_1^{(1)}\in\mathbb{R}^{M\times Td}$, $A_1^{(2)}\in\mathbb{R}^{Td \times Td}$ are random matrices and $\alpha_1,\dots,\alpha_d$ are random variables. The matrix $A_1^{(2)}$ is given by
    \begin{equation}\label{samplingscheme1UT}
        A_1^{(2)} = \left[\begin{array}{cc}
            0_{d\times (T-1)d} & 0_{d \times d}  \\[2mm]
            U & 0_{(T-1)d \times d}
        \end{array}\right],
    \end{equation}
    where $U\in\mathbb{R}^{(T-1)d \times (T-1)d}$ is an upper triangular random matrix. Let $\mu_1$ be an absolutely continuous one-dimensional distribution with full support on $\mathbb{R}$. We assume that the random entries of $A_1^{(1)},U, \xi_1^{(1)}$ as well as the random variables $\alpha_1,\dots,\alpha_d$ are i.i.d. under $\mu_1$. 
\end{samplingscheme}

\begin{remark}\label{re:remarkinv}
    Note that det$(U) = 0$ with probability zero, because the entries of $U$ are i.i.d. with an absolutely continuous distribution $\mu_1$. In other words, the matrix $U$ is almost surely invertible. 
\end{remark}

For the remainder of the paper, we make the following assumption on the activation function $\sigma$.

\begin{assumption}\label{ass:assumactivation}
    The activation function $\sigma$ in \eqref{eq:reservoirsystem} is continuous, bounded, non-polynomial, injective and satisfies $\sigma(0) = 0$.
\end{assumption}

Prominent activation functions satisfying the properties stated in Assumption \ref{ass:assumactivation} are the Tanh or a shifted Sigmoid function. Under Assumption \ref{ass:assumactivation}, the Sampling Scheme \ref{samplingscheme1} naturally induces a block structure for the process $\text{RS}(x)$ defined in \eqref{eq:reservoirsystem}. In particular, for every $t=1,\dots,T$ it holds
\begin{equation}\label{blockstructureZ^2scheme1}
    \left[\begin{array}{c}
         \text{RS}^{(1)}_t(x)  \\[1.5mm]
         \text{RS}^{(2)}_t(x)
    \end{array}\right] = 
    \left[\begin{array}{c}
         \text{RS}^{(1)}_{t-1}(x)  \\[1.5mm]
         \text{RS}^{(2)}_{t-1}(x)
    \end{array}\right] + 
    \left[\begin{array}{c}
         \sigma(A_1^{(1)}\text{RS}^{(2)}_{t-1}(x) + \xi_1^{(1)})  \\[1.5mm]
         \sigma(A_1^{(2)}\text{RS}^{(2)}_{t-1}(x))
    \end{array}\right] + 
    \left[\begin{array}{c}
         0  \\[1.5mm]
         \sum_{i=1}^d \sigma(\alpha_i) e_i x_t^i
    \end{array}\right].
\end{equation}
    
The following lemma shows that under the Sampling Scheme \ref{samplingscheme1} the $\mathbb{R}^{Td}$-valued random vector $\text{RS}^{(2)}_{t}(x)$, $t=1,\dots,T$, can almost surely be written as an injective mapping of $(x_1,\dots,x_t)$. 

\begin{lemma}\label{lemma:injective1}
    Following the Sampling Scheme \ref{samplingscheme1}, for every $t=1,\dots,T$ there exists almost surely a unique continuous injective function $G_t:(\mathbb{R}^{d})^t\to (\mathbb{R}^{d})^T$ such that
    \begin{equation}\label{injectiveZt}
         \text{RS}^{(2)}_{t}(x) = G_t(x_1,\dots,x_t).
    \end{equation}
\end{lemma}

\begin{proof}
    The proof consists of two parts. In the first part, we prove that for every $t=1,\dots,T$, the function $G_t$ satisfying equation \eqref{injectiveZt} is of the following form 
    \begin{equation}\label{blockformG}
        G_t(x_1,\dots,x_t) = \left[\begin{array}{c}
             g_t^{(1)}(x_1,\dots,x_t)\\
             \vdots \\
             g_t^{(t)}(x_1) \\[1mm]
             0_{(T-t)d\times 1}
         \end{array}\right]
    \end{equation}
    for some unique component 
    functions $g^{(k)}_t:(\mathbb{R}^d)^{t+1-k} \to \mathbb{R}^d$, $k= 1,\dots,t$. In the second part of the proof, we show that the functions $g_t^{(k)}(x_1,\dots,x_{t-k},\cdot)$ are almost surely injective for fixed $x_1,\dots,x_{t-k}$, which implies the injectivity of $G_t$.

    \textit{First part:} We prove equality \eqref{injectiveZt} with $G_t$ satisfying \eqref{blockformG} by induction over $t=1,\dots,T$. For time $t=1$ it follows from \eqref{blockstructureZ^2scheme1} that 
    \begin{equation}\label{inductionstart1}
         \text{RS}^{(2)}_{1}(x) = 
        \left[\begin{array}{c}
             \sigma(\alpha_1) x_1^1\\
             \vdots \\
             \sigma(\alpha_d) x_1^d \\[1mm]
             0_{(T-1)d\times 1} 
        \end{array}\right] =: 
        \left[\begin{array}{c}
             g_1^{(1)}(x_1) \\[1mm]
             0_{(T-1)d\times 1} 
        \end{array}\right],
    \end{equation}
    as RS$_0(x) = 0$, where $g^{(1)}_1$ is almost surely well-defined and linear in $x_1$. So representation \eqref{blockformG} holds for $t=1$. We now prove by induction that if \eqref{injectiveZt} holds with $G_t$ satisfying \eqref{blockformG} for $t=1,\dots,T-1$, then it holds for $t+1$. For this purpose, we observe that the upper triangular matrix $U$ from \eqref{samplingscheme1UT} can be written as 
    \begin{equation}\label{upblockstructure}
        U = \begin{bNiceMatrix}
            U_1    & B_{1,1} & \Cdots  & & B_{1,T-2} \\
            0      & \Ddots  & \Ddots &  & \Vdots    \\
            \Vdots & \Ddots  & &        &           \\
                   &         & &        & B_{T-2,1} \\
            0      & \Cdots  & & 0      & U_{T-1}
        \end{bNiceMatrix}
    \end{equation}
    for some upper triangular random matrices $U_1,\dots,U_{T-1}\in\mathbb{R}^{d\times d}$ and some arbitrary random matrices $B_{i,j}\in\mathbb{R}^{d\times d}$, $i = 1,\dots,T-2$, $j = 1,\dots, T-1-i$, whose entries are i.i.d. under $\mu_1$. From \eqref{blockstructureZ^2scheme1}, \eqref{upblockstructure} and the assumption of our induction it follows
    \begin{align*}
        \text{RS}^{(2)}_{t+1}(x) =& \left[\begin{array}{c}
        g_t^{(1)}(x_1,\dots,x_t) +  \sum_{i=1}^d \sigma(\alpha_i)\tilde{e}_i x^i_{t+1}\\[2mm]
        g_t^{(2)}(x_1,\dots,x_{t-1}) + \sigma\left(U_1 g_t^{(1)}(x_1,\dots,x_t) + \sum_{k=1}^{t-1} B_{1,k} g_t^{(k+1)}(x_1,\dots,x_{t-k})\right)\\[2mm]
        \vdots\\[2mm]
        g_t^{(t)}(x_1) + \sigma\left(U_{t-1}g_t^{(t-1)}(x_1,x_2) + B_{t-1,1}g_t^{(t)}(x_1)\right)\\[2mm]
        \sigma\left(U_t g_t^{(t)}(x_1)\right)\\[2mm]
        0_{(T-(t+1))d\times 1}
        \end{array}\right],
    \end{align*}
    where $\tilde{e}_i$ denots the $i$-th canonical basis vector of $\mathbb{R}^d$. Then \eqref{blockformG} holds for the following choices of the functions $g_{t+1}^{(k)}:(\mathbb{R}^d)^{t+2-k}\to\mathbb{R}^d$, $k=1,\dots,t+1$, given by 
    \begin{equation}\label{g^(t+1)}
        g_{t+1}^{(1)}(x_1,\dots,x_{t+1}) := g_t^{(1)}(x_1,\dots,x_t) + \sum_{i=1}^d \sigma(\alpha_i)\tilde{e}_i x_{t+1}^{i},
    \end{equation}

    \begin{align}\label{g^(k)}
        &g_{t+1}^{(k)}(x_1,\dots,x_{t+2-k}) - g_t^{(k)}(x_1,\dots,x_{t+1-k}) \\[2mm]
        := & \sigma\left(U_{k-1}g_{t}^{(k-1)}(x_1,\dots,x_{t+2-k}) + \sum_{j=1}^{t+1-k}B_{k-1,j}g_t^{(j+k-1)}(x_1,\dots,x_{t+2-k-j})\right)\nonumber
    \end{align}
    for $k=2,\dots, t$ and
    \begin{equation}\label{g^(1)}
        g_{t+1}^{(t+1)}(x_1) := \sigma\left(U_t g_t^{(t)}(x_1)\right).
    \end{equation}
    
    Note that the construction of the component functions is unique and thus concludes the first step of the proof.

    \textit{Second part:} We prove that the component functions in \eqref{blockformG} are almost surely injective as functions only of their last variable, i.e. for all $t=1,\dots,T$ and $k=1,\dots,t$ the function
    \begin{equation}\label{eq:bijectiveinlast}
        g_t^{(k)}(x_1,\dots, x_{t-k},\cdot):z\mapsto g_t^{(k)}(x_1,\dots, x_{t-k}, z)
    \end{equation}
    is almost surely injective. We prove this statement by induction over $t$. For time $t=1$ we have only one component function defined in \eqref{inductionstart1}. The events $\{\alpha_i = 0\}$, $i=1,\dots,d$, have zero probability due to the absolute continuity of the distribution $\mu_1$. Therefore, the component function from \eqref{inductionstart1} is almost surely injective, since it holds $\sigma(0)=0$ by Assumption \ref{ass:assumactivation}. 
    
    For the induction step $t\mapsto t+1$ we first consider $k=t+1$, i.e. the function $g_{t+1}^{(t+1)}$ is defined in \eqref{g^(1)}. By assumption, the function $g_t^{(t)}$ is almost surely injective. Moreover, by Assumption \ref{ass:assumactivation} the activation function $\sigma$ is injective and $U_t$ is almost surely invertible by the same reasoning as in Remark \ref{re:remarkinv}. Therefore, it follows that $g_{t+1}^{(t+1)}$ is almost surely injective. For $k=t$ we obtain 
    $$
        g_{t+1}^{(t)}(x_1,x_2) = g_t^{(t)}(x_1) + \sigma\left(U_{t-1} g_t^{(t-1)}(x_1,x_2) + B_{t-1,1}g_t^{(t)}(x_1)\right).
    $$
    For a fixed $x_1$, we can uniquely determine $x_2$ in $g_{t+1}^{(t)}$ again using that $\sigma$ is injective and $U_{t-1}$ is (almost surely) invertible. With fixed $x_1, x_2$ we can uniquely determine $x_3$ in $g_{t+1}^{(t-1)}$. Therefore the result follows by backward induction for all $k=1,\dots,t+1$. This also concludes the induction on $t$, i.e. the function $g_{t}^{(k)}$ is almost surely injective in the last variable. Note that this is sufficient to guarantee that the function $G_t$ defined in \eqref{blockformG} is almost surely injective in all variables for all $t=1,\dots,T$. In addition, all component functions are continuous as compositions of continuous functions, which concludes the proof.
\end{proof}

Before we prove the universality of linear readouts of $\Delta \text{RS}_T(x)$, we state a universal approximation result of random feature neural networks as defined in Definition \ref{def:randomfeatureDNNs} for continuous functions on a compact domain. The result is very similar to Theorem 2.4.5 in \cite{Hart2020}, where the authors consider activation functions that are $\ell$-finite for some $\ell\in \mathbb{N}_0$ to prove approximation results for functions in $\mathcal{C}^{\ell}$. We consider non-polynomial activation functions without integrability assumptions so that we can include the Tanh or Sigmoid function in our analysis.

\begin{proposition}[Universality of random feature neural networks]\label{prop:universalityrandomDNNs}
    Let $d\in\mathbb{N}$, $M\in\mathbb{N}$, $K\subseteq\mathbb{R}^d$ a compact set and $f\in\mathcal{C}(K)$. If the activation function $\sigma$ is continuous and non-polynomial, $A$ is a $\mathbb{R}^{M\times d}$-valued random matrix and $\xi$ is a $\mathbb{R}^M$-valued random vector both with i.i.d. entries with a common continuous distribution, then for any $\alpha\in(0,1)$ and $\varepsilon>0$ there exists a choice for $M$ such that with probability greater than $\alpha$, there exists a vector $w\in\mathbb{R}^M$ such that the corresponding random feature neural network $\Psi^{A,\xi}_w:K\to\mathbb{R}$ defined by
    $$
        \Psi^{A,\xi}_w(x) = w^T \sigma(Ax + \xi)
    $$
    satisfies
    $$
        \Vert f - \Psi^{A,\xi}_w \Vert_\infty < \varepsilon,
    $$
    where $\Vert \cdot \Vert_\infty$ denotes the supremum norm on $K$.
\end{proposition}

\begin{proof}
    The proof follows the same steps as in \cite{Hart2020} but applies Theorem 1 in \cite{Leshno1993} instead of a universality result for functions in $\mathcal{C}^k(K)$ with $k\geq 1$. 
\end{proof}

Based on Proposition \ref{prop:universalityrandomDNNs} we can now state the following universality result for the terminal difference of the randomised signature $\text{RS}(x)$ defined in \eqref{eq:reservoirsystem}. 

\begin{proposition}\label{prop:universalityreservoir1}
    Let $K\subseteq \mathbb{R}^{d}$ be a compact subset and $f\in\mathcal{C}(K^{T-1})$. The activation function $\sigma$ satisfies Assumption \ref{ass:assumactivation}. Then under the Sampling Scheme \ref{samplingscheme1}, for any $\alpha\in(0,1)$ and $\varepsilon>0$ there exists a $M\in\mathbb{N}$ such that with probability greater than $\alpha$, there exists $\tilde{w}\in\mathbb{R}^{M}$ such that
    $$
        \sup_{x_1,\dots,x_{T-1}\in K}\vert f(x_1,\dots,x_{T-1}) - \langle w, \Delta \text{RS}_T(x_1,\dots,x_{T-1})\rangle\vert < \varepsilon
    $$
    for the output vector $w := [\begin{array}{cc} \tilde{w}^T & 0_{1\times Td}\end{array}]^T$ and $\langle\cdot,\cdot\rangle$ denoting the standard scalar product on $\mathbb{R}^{M\times Td}$.
\end{proposition}

\begin{proof}
    Let $\alpha\in(0,1)$ and $\varepsilon>0$. Consider the Sampling Scheme \ref{samplingscheme1}. By Lemma \ref{lemma:injective1} for every $t=1,\dots,T$ there exists an almost surely injective function $G_t:(\mathbb{R}^d)^t \to (\mathbb{R}^d)^T$ such that $G_t(x_1,\dots,x_t) = \text{RS}_t^{(2)}(x)$ as defined in \eqref{blockstructureZ^2scheme1}. So in particular, there exists a function $G_t^{-1}:\text{Im}(G_t)\to (\mathbb{R}^d)^t$ such that
    \begin{equation}\label{eq:inverseG}
        G^{-1}_t(\text{RS}^{(2)}_t(x)) = G^{-1}_t(G_t(x_1,\dots,x_t)) = (x_1,\dots,x_t),
    \end{equation}
    where Im$(G_t)$ denotes the image of $G_t$.
    On the other hand, by \eqref{blockstructureZ^2scheme1} it holds that 
    \begin{equation}\label{eq:deltaZ^1}
        \Delta \text{RS}^{(1)}_T(x) = \sigma(A_1^{(1)}\text{RS}^{(2)}_{T-1}(x) + \xi_1^{(1)}).
    \end{equation}

    By Lemma \ref{lemma:injective1} the function $G_{T-1}$ is almost surely continuous, hence in particular the set $\bar{K} := G_{T-1}(K^{T-1})$ is compact. Moreover, the function $f\circ G^{-1}_{T-1}$ is almost surely continuous as a composition of two continuous functions. Therefore, by Proposition \ref{prop:universalityrandomDNNs} we can choose $M\in\mathbb{N}$ such that with probability greater than $\alpha$, there exists $\tilde{w}\in\mathbb{R}^M$ such that with \eqref{eq:inverseG} it holds 
    \begin{align*}
        &\vert f(x_1,\dots,x_{T-1}) - \tilde{w}^T \sigma(A_1^{(1)}\text{RS}^{(2)}_{T-1}(x) + \xi_1^{(1)})\vert \\[2mm] 
        =&\vert \left(f\circ G_{T-1}^{-1}\right)(\text{RS}^{(2)}_{T-1}(x)) - \tilde{w}^T \sigma(A_1^{(1)}\text{RS}^{(2)}_{T-1}(x) + \xi_1^{(1)})\vert \\[2mm]
        <&\varepsilon
    \end{align*}
    for all $x_1,\dots,x_{T-1}\in K$. So choosing $w := [\begin{array}{cc} \tilde{w}^T & 0_{1\times Td}\end{array}]^T$ yields by \eqref{blockstructureZ^2scheme1} for all $x_1,\dots,x_{T-1}\in K$
    \begin{align*}
        &\vert f(x_1,\dots,x_{T-1}) - \langle w,\Delta \text{RS}_T(x) \rangle\vert \\[2mm]
        =& \vert f(x_1,\dots,x_{T-1}) - \tilde{w}^T \sigma(A_1^{(1)}\text{RS}^{(2)}_{T-1}(x) + \xi_1^{(1)})\vert \\[2mm]
        <& \varepsilon,
    \end{align*}
    which concludes the proof. Note that the set $\bar{K}$ depends on $A_1^{(2)}$ and $\alpha_1,\dots,\alpha_d$. Therefore, it is worth noting that the entries of $A_1^{(2)}$ and the random variables $\alpha_1,\dots,\alpha_d$ are independent of the entries of $A_1^{(1)}$ and $\xi_1^{(1)}$ so that we could apply Proposition \ref{prop:universalityrandomDNNs} independently of $\bar{K}$. 
\end{proof}

Note that in Proposition \ref{prop:universalityreservoir1} the approximated function $f$ does not take the terminal value of $x$ as an argument. However, we can prove a universality result if the function $f$ is of the following form
\begin{equation}\label{eq:specialf}
    f(x_1,\dots,x_T) = g(x_1,\dots,x_{T-1}) + \sum_{i=1}^d h_i(x_1,\dots,x_{T-1})x^i_T
\end{equation}

for continuous $g,h:(\mathbb{R}^d)^{T-1}\to\mathbb{R}$. In order to prove a approximation result for a function of the form \eqref{eq:specialf}, we introduce the following alternative sampling scheme.

\begin{samplingscheme}\label{samplingscheme2}
    Let $N:= M + M_1 + \dots + M_d + Td$ for some $M,M_1,\dots,M_d\in\mathbb{N}$ and denote by $\tilde{M}_j := \sum_{i=1}^j M_i$, $j=1,\dots,d$. We consider the following sampling scheme for the hidden weights of $\eqref{eq:reservoirsystem}$. For $i=1,\dots,d$ let  
    \begin{equation*}
        A_1 = \left[\begin{array}{ccc}
             0_{M\times M} & 0_{M\times \tilde{M}_d} & A^{(1)}_1 \\
             0_{\tilde{M}_d\times M} & 0_{\tilde{M}_d \times\tilde{M}_d} & 0_{\tilde{M}_d\times Td} \\
             0_{Td\times M} & 0_{Td\times\tilde{M}_d} & A^{(2)}_1
        \end{array}\right],\; \xi_1 = \left[\begin{array}{c}
             \tilde{\xi}_1\\
             0_{\tilde{M}_d\times 1} \\
             0_{T\times 1}
        \end{array}\right],
        \end{equation*}
    \begin{equation*}
        A^i_2 = \left[\begin{array}{ccc}
             0_{M + \tilde{M}_{i-1}\times M} & 0_{M + \tilde{M}_{i-1}\times \tilde{M}_d} & 0_{M + \tilde{M}_{i-1}\times Td} \\
             0_{M_i\times M} & 0_{M_i \times \tilde{M}_d} & \tilde{A}_2^{i} \\
             0_{Td\times M}& 0_{Td\times \tilde{M}_d}& 0_{ Td\times Td}
        \end{array}\right],\;  \xi^i_2 = \left[\begin{array}{c}
             0_{M + \tilde{M}_{i-1}\times 1}\\
             \tilde{\xi}_2^{i} \\
             \alpha_i e_i
        \end{array}\right],
    \end{equation*}
    where $e_i$ denotes the $i$-th canonical basis vector of $\mathbb{R}^{Td}$, $A_1^{(1)}\in\mathbb{R}^{M\times Td}$, 
    $\tilde{\xi}_1\in\mathbb{R}^{M}$, $A_1^{(2)}\in\mathbb{R}^{Td \times Td}$, $\tilde{A}_2^{i}\in\mathbb{R}^{M_i \times Td}$, $\tilde{\xi}^{i}_2\in\mathbb{R}^{M_i}$ are random matrices and $\alpha_1,\dots,\alpha_d$ random variables. In particular, the matrix $A_1^{(2)}$ is of the form 
    \begin{equation}\label{samplingscheme2UT}
        A_1^{(2)} = \left[\begin{array}{cc}
            0_{d\times (T-1)d} & 0_{d \times d}  \\[2mm]
            U & 0_{(T-1)d \times d}
        \end{array}\right],
    \end{equation}
    where $U\in\mathbb{R}^{(T-1)d \times (T-1)d}$ is an upper triangular random matrix. Let $\mu_2$ be an absolutely continuous one-dimensional distribution with full support on $\mathbb{R}$. We assume that the random entries of $A_1^{(1)}, A_1^{(2)}, \tilde{\xi}_1$, $\tilde{A}_2^{i}, \tilde{\xi}^{i}_2$, $U$ as well as the random variables $\alpha_1,\dots,\alpha_d$ are i.i.d. under $\mu_2$.
\end{samplingscheme}

As in \eqref{blockstructureZ^2scheme1} also the Sampling Scheme \ref{samplingscheme2} induces a block structure for the process $\text{RS}(x)$ defined in \eqref{eq:reservoirsystem}. In particular, for every $t=1,\dots,T$ it holds
\begin{equation}\label{eq:blockstructureZ^2scheme2}
    \left[\begin{array}{c}
         \text{RS}^{(1)}_{t}(x) \\[1.5mm]
         \text{RS}^{(2)}_{t}(x)\\[1.5mm]
         \vdots \\[1.5mm]
         \text{RS}^{(d+1)}_{t}(x) \\[1.5mm]
         \text{RS}^{(d+2)}_{t}(x)
    \end{array}\right] = 
    \left[\begin{array}{c}
         \text{RS}^{(1)}_{t-1}(x) \\[1.5mm]
         \text{RS}^{(2)}_{t-1}(x) \\[1.5mm]
         \vdots \\[1.5mm]
         \text{RS}^{(d+1)}_{t-1}(x) \\[1.5mm]
         \text{RS}^{(d+2)}_{t-1}(x)
    \end{array}\right] + 
    \left[\begin{array}{c}
         \sigma(A_1^{(1)}\text{RS}^{(d+2)}_{t-1}(x) + \tilde{\xi}_1)\\[1.5mm]
         0 \\[1.5mm]
         \vdots\\[1.5mm]
         0 \\[1.5mm]
         \sigma(A_1^{(2)}\text{RS}^{(d+2)}_{t-1}(x))
    \end{array}\right] + 
    \left[\begin{array}{c}
         0  \\[1.5mm]
         \sigma(\tilde{A}_2^{1}\text{RS}^{(d+2)}_{t-1}(x) + \tilde{\xi}_2^{1})x^1_t \\[1.5mm]
         \vdots\\[1.5mm]
         \sigma(\tilde{A}_2^{d}\text{RS}^{(d+2)}_{t-1}(x) + \tilde{\xi}_2^{d})x^d_t \\[1.5mm]
        \sum_{i=1}^d \sigma(\alpha_i) e_i x_t^i
    \end{array}\right].
\end{equation}
Similar to Sampling Scheme \ref{samplingscheme1}, we can show that the last entry $\text{RS}_{t}^{(d+2)}$ in \eqref{eq:blockstructureZ^2scheme2} can be written as an injective function on the values of $x$ up to time $t=1,\dots,T$.

\begin{lemma}\label{lemma:injective2}
    Following the Sampling Scheme \ref{samplingscheme2}, for every $t=1,\dots,T$ there exists an almost surely unique continuous injective function $G_t:(\mathbb{R}^d)^t\to (\mathbb{R}^d)^T$ such that
    \begin{equation}\label{injectiveZt2}
         \text{RS}^{(d+2)}_{t}(x) = G_t(x_1,\dots,x_t).
    \end{equation}
\end{lemma}

\begin{proof}
    This proof is exactly the same as the one of Lemma \ref{lemma:injective1}.
\end{proof}
Based on Lemma \ref{lemma:injective2} we obtain the following universal approximation result for the terminal difference of the randomised signature for continuous functions of the form \eqref{eq:specialf}.

\begin{proposition}\label{prop:universalityreservoir2}
    Let $f\in\mathcal{C}(K^T)$ be a function of the form \eqref{eq:specialf}. The activation function $\sigma$ satisfies Assumption \ref{ass:assumactivation}. Then under the Sampling Scheme \ref{samplingscheme2}, for any $\alpha\in(0,1)$ and $\varepsilon>0$ there exist $M\in\mathbb{N}$ and $M_1,\dots,M_d\in\mathbb{N}$ such that with probability greater than $\alpha$, there exist $\tilde{w}\in\mathbb{R}^{M}$ and $\tilde{w_i}\in\mathbb{R}^{M_i}$, $i=1,\dots,d$, such that
    $$
        \sup_{x_1,\dots,x_T \in K}\vert f(x_1,\dots,x_T) - \langle w, \Delta \text{RS}_{T}(x_1,\dots,x_T)\rangle\vert < \varepsilon
    $$
    for the output vector $w = [\begin{array}{ccccc} \tilde{w}^T & \tilde{w_1}^T & \hdots & \tilde{w_d}^T & 0_{1\times Td}\end{array}]^T$.
\end{proposition}

\begin{proof}
    Let $\alpha\in(0,1)$ and $\varepsilon > 0$. Consider the Sampling Scheme \ref{samplingscheme2}. As in the proof of Proposition \ref{prop:universalityreservoir1} we have
    \begin{equation}\label{eq:inverseG2}
        G^{-1}_t(\text{RS}^{(d+2)}_t(x)) = G^{-1}_t(G_t(x_1,\dots,x_t)) = (x_1,\dots,x_t)
    \end{equation}
    for any $x\in K^T$. For the block components of $\text{RS}_T(x)$ following Sampling scheme \ref{samplingscheme2} it holds using \eqref{eq:blockstructureZ^2scheme2} 
    \begin{equation}\label{eq:deltaZ^12}
        \Delta \text{RS}^{(1)}_T(x) = \sigma(A_1^{(1)}\text{RS}^{(d+2)}_{T-1}(x) + \tilde{\xi}_1)
    \end{equation}
    and for $i=1,\dots,d$ 
    \begin{equation}\label{eq:deltaZ^i2}
        \Delta \text{RS}^{(i+1)}_T(x) = \sigma(\tilde{A}_2^{i}\text{RS}^{(d+2)}_{T-1}(x) + \tilde{\xi}_2^i)x^i_t.
    \end{equation}
    Since $G_{T-1}$ from Lemma \ref{lemma:injective2} is almost surely continuous, it follows that the set $\bar{K} := G_{T-1}(K^{T-1})$ is compact. Moreover, the functions $g\circ G_{T-1}^{-1}$, $h_i\circ G_{T-1}^{-1}$, $i=1,\dots,d$, are almost surely continuous. Therefore, by Proposition \ref{prop:universalityrandomDNNs} there exist $M,M_1,\dots,M_d\in\mathbb{N}$ such that with probability greater than $\alpha$, there exist $\tilde{w}\in\mathbb{R}^{M}$ and $\tilde{w}_1\in\mathbb{R}^{M_1},\dots, \tilde{w}_d\in\mathbb{R}^{M_d}$ such that with Lemma \ref{lemma:injective2} it holds
    \begin{align}\label{eq:boundapprox1}
        &\sup_{x_1,\dots,x_{T-1}\in K}\vert g(x_1,\dots,x_{T-1}) - \tilde{w}^T\sigma(A_1^{(1)}\text{RS}^{(d+2)}_{T-1}(x) + \tilde{\xi}_1) \vert \nonumber\\[2mm] =& \sup_{x_1,\dots,x_{T-1}\in K}\vert (g\circ G_{T-1}^{-1})(\text{RS}^{(d+2)}_{T-1}(x)) - \tilde{w}^T\sigma(A_1^{(1)}\text{RS}^{(d+2)}_{T-1}(x) + \tilde{\xi}_1) \vert
        < \frac{\varepsilon}{2}
    \end{align}
    and for every $i=1,\dots,d$
    \begin{align}\label{eq:boundapprox2}
        &\sup_{x_1,\dots,x_{T-1}\in K}\vert h_i(x_1,\dots,x_{T-1}) - \tilde{w}_i^T \sigma(\tilde{A}^i_2 \text{RS}_{T-1}^{(d+2)}(x) + \tilde{\xi}^i_2)\vert \nonumber\\[2mm]
        =& \sup_{x_1,\dots,x_{T-1}\in K}\vert (h_i\circ G_{T-1}^{-1})(\text{RS}_{T-1}^{(d+2)}(x))- \tilde{w}_i^T \sigma(\tilde{A}^i_2 \text{RS}_{T-1}^{(d+2)}(x) + \tilde{\xi}^i_2)\vert 
        < \frac{\varepsilon}{2d \sup \{\vert x\vert \mid x \in K\}},
    \end{align}
    Therefore, by choosing $w = [\begin{array}{ccccc} \tilde{w}^T & \tilde{w}_1^T & \hdots & \tilde{w}_d^T & 0_{1\times Td}\end{array}]^T$ we obtain by \eqref{eq:deltaZ^12} and \eqref{eq:deltaZ^i2} that
    $$
        \langle w,\Delta \text{RS}_T(x) \rangle = \tilde{w}^T\sigma(A_1^{(1)}\text{RS}^{(d+2)}_{T-1}(x) + \tilde{\xi}_1) + \sum_{i=1}^d \tilde{w}_i^T \sigma(\tilde{A}^i_2 \text{RS}_{T-1}^{(d+2)}(x) + \tilde{\xi}^i_2)x^i_T.
    $$
    In particular, combining \eqref{eq:boundapprox1} with \eqref{eq:boundapprox2} yields
    \begin{align*}
        &\sup_{x_1,\dots,x_{T}\in K}\vert f(x_1,\dots,x_{T}) - \langle w,\Delta \text{RS}_T(x) \rangle \vert \\[2mm] 
        \leq&\sup_{x_1,\dots,x_{T-1}\in K}\vert g(x_1,\dots,x_{T-1}) - \tilde{w}^T\sigma(A_1^{(1)}\text{RS}^{(d+2)}_{T-1}(x) + \tilde{\xi}_1) \vert \\[2mm]
        \phantom{....}&+ \sum_{i=1}^d \sup_{x_1,\dots,x_{T-1}\in K}\vert h_i(x_1,\dots,x_{T-1}) - \tilde{w}_i^T \sigma(\tilde{A}^i_2 \text{RS}_{T-1}^{d+2}(x) + \tilde{\xi}^i_2)\vert \sup \{\vert x\vert \mid x \in K\} \\[2mm]
        \leq& \frac{\varepsilon}{2} + d \frac{\varepsilon}{2 d \sup \{\vert x\vert \mid x \in K\}}\sup \{\vert x\vert \mid x \in K\} = \varepsilon,
    \end{align*}
    which concludes the proof.
\end{proof}

\section{Time series generation using the randomised signature}\label{Section3}
Consider the compact set $K\subseteq \mathbb{R}^d$ introduced in Section \ref{Section2}. In the following, we denote by $\mathcal{X}:= K\times \dotsb \times K \subseteq (\mathbb{R}^d)^T$. Let $X = (X_t)_{t=1,\dots,T}$ be a $\mathcal{X}$-valued discrete-time stochastic process defined on the probability space with unknown distribution $\mathbb{P}^{\text{real}} := \mathbb{P}\circ X^{-1}$. In this section, we want to generate trajectories of a $\mathcal{X}$-valued stochastic process $\hat{X}^\theta = (\hat{X}^\theta_t)_{t=1,\dots,T}$ with distribution $\mathbb{P}^{\text{fake}}_\theta := \mathbb{P}\circ (\hat{X}^\theta)^{-1}$ such that $d_\text{path}(\mathbb{P}^{\text{real}}, \mathbb{P}^{\text{fake}})$ is small, where $d_\text{path}$ is a (pseudo-)metric on the space of probability measures on $\mathcal{X}$ denoted by $\mathcal{P}(\mathcal{X})$. We will introduce a specific novel choice for $d_\text{path}$ based on randomised signature later this section. In addition, we define $\hat{X}^\theta$ based on randomised signature as follows.

Consider a $n$-dimensional $\mathbb{F}$-Brownian motion $W=(W^1_t,\dots, W^n_t)_{t\in[0,T]}$ and a random vector $V\sim\mathcal{N}_m(0,\text{Id})$ both defined on the probability space $(\Omega,\mathcal{F},\mathbb{F},\mathbb{P})$ for some $n,m\in\mathbb{N}$ and independent of each other. For some $D\in\mathbb{N}$ we consider the $D$-dimensional stochastic process $R = (R_t)_{t = 1,\dots,T}$ and a $d$-dimensional stochastic process $X^\theta =(X^\theta_t)_{t=1,\dots,T}$ defined as
\begin{equation}\label{eq:generatorSDE}
    \begin{cases}
        R_1 = \Psi^\theta(V),\\[2mm]
        R_t = R_{t-1} + \sigma(\rho_1^\theta B_1 R_{t-1} + \rho_2^\theta \lambda_1) + \sum_{i=1}^n \sigma(\rho_3^\theta B^i_2 R_{t-1} + \rho_4^\theta \lambda^i_2)\rho_5^\theta(W^i_{t} - W^i_{t-1}),\\[2mm]
        X^\theta_t = A^\theta_t R_t + \beta^\theta_t, \; t = 1,\dots, T
    \end{cases}
\end{equation}
for a neural network $\Psi^\theta:\mathbb{R}^m \to \mathbb{R}^D$, random matrices $B_1, B^1_2,\dots, B^n_2 \in\mathbb{R}^{D\times D}$, random vectors $\lambda_1, \lambda^1_2,\dots,\lambda^w_2 \in\mathbb{R}^{D}$, parameters $\rho_1^\theta,\dots, \rho_5^\theta\in\mathbb{R}$, $A^\theta_1,\dots,A^\theta_T\in\mathbb{R}^{d\times D}$, $\beta^\theta_1,\dots,\beta^\theta_T\in\mathbb{R}^d$ and some activation function $\sigma:\mathbb{R} \to\mathbb{R}$, which is applied component-wise. The entries of the random matrices are i.i.d.\hspace{-1mm} with some absolutely continuous distribution $\mu_3$ with full support. Moreover, we assume that the random entries are independent of the process $X$ and of $W,V$. The superscript $\theta$ implies that the weights of the neural network $\Psi^\theta$, the parameters $\rho_1^\theta,\dots,\rho_5^\theta$ and the readouts $A^\theta_1,\dots,A^\theta_N,\beta^\theta_1,\dots,\beta^\theta_N$ are trainable. One can think of the process $R$ as a dynamical reservoir system evolving over time. Comparing to \eqref{eq:reservoirsystem}, we see that $R$ is a variant of the randomised signature of $(W_t - W_{t-1})_{t=1,\dots,T}$ with randomly initialised state $R_1$. The value $X^{\theta}_t$, $t=1,\dots,T$, is the output of a linear readout from the state of $R$ at time $t$. For some $r>0$ we consider the following projection defined for any $x\in\mathbb{R}^d$ as
\begin{equation*}
    \pi_r(x) := 
    \begin{cases}
        x,\text{ if } \Vert x \Vert \leq r,\\
        r\frac{x}{\Vert x\Vert}, \text{ if } \Vert x \Vert >r.
    \end{cases}
\end{equation*}
Finally, we define $\hat{X}^\theta_t := \pi_r(X^\theta_t)$ for any $t=1,\dots,T$, where the parameter $r$ is chosen such that $\hat{X}^\theta\in\mathcal{X}$ almost surely. From a practical point of view, the projection can be omitted in the numerical implementation if the subset $\mathcal{X}$ is considered large enough. 

We can formulate the following training problem for generating synthetic trajectories of $X$ using $\hat{X}^\theta$ as 
\begin{equation}\label{eq:trainingproblem}
    \min_\theta d_\text{path}(\mathbb{P}^{\text{real}},\mathbb{P}^{\text{fake}}_\theta).
\end{equation}
Prior to the training procedure, the entries of $B_1,B^1_2,\dots,B^w_2$ and $\lambda_1,\lambda^1_2,\dots,\lambda^w_2$ are drawn from the absolutely continuous distribution $\mu_3$ and then fixed. Hence, the only remaining source of randomness is the Brownian motion $W$ and random vector $V$. In the following, we propose a novel choice for $d_\text{path}$ for the problem \eqref{eq:trainingproblem}. 

\subsection{The randomised signature Wasserstein-1 (RS-$W_1$) metric}
In this section, we introduce a modified version of the well-known Wasserstein-1 metric first introduced in \cite{Kantorovich1960}. Recall that for two probability measures $\mu,\nu\in\mathcal{P}(\mathcal X)$ the dual representation of the $W_1$ distance of Kantorovich and Rubinstein is given by 
\begin{equation}\label{eq:dual_representation_W1}
    W_1(\mu,\nu) = \sup_{\Vert f\Vert_L \leq 1} \mathbb{E}_{\mu}[f(X)] - \mathbb{E}_{\nu}[f(X)],
\end{equation}
where $\Vert\cdot\Vert_L$ denotes the Lipschitz norm defined for any $f:\mathcal{X}\to\mathbb{R}$ as 
\begin{equation*}\label{eq:lipschitz_norm}
    \Vert f\Vert_L := \sup_{x,y\in \mathcal{X}}\frac{\vert f(x) - f(y)\vert}{\Vert x-y\Vert_2}.
\end{equation*}
In the last section, we have seen that the randomised signature is a universal feature map. In other words, we can approximate continuous functions on $\mathcal{X}$ by linear readouts of the terminal difference of the randomised signature. Proposition \ref{prop:universalityreservoir1} and \ref{prop:universalityreservoir2} state that for a large class of functions $f\in\mathcal{C}(\mathcal{X})$ there exist with an arbitrary large probability a dimension $N$ for the randomised signature and a vector $w\in\mathbb{R}^N$ such that $\langle w, \Delta \text{RS}_T(\cdot)\rangle:\mathcal{X} \to \mathbb{R}$ approximates $f$ up to an arbitrarily small error. In particular, this motivates that the Lipschitz-continuous function considered in \eqref{eq:dual_representation_W1} could be approximated as well. Therefore, we propose the following randomised signature Wasserstein-1 metric 
$$
    \text{RS-}W_1(\mu,\nu) := \sup_{w \in\mathbb{R}^M,\; \Vert w\Vert \leq 1}\mathbb{E}_{\mu}[\langle w, \Delta \text{RS}_T(X)\rangle] - \mathbb{E}_{\nu}[\langle w, \Delta \text{RS}_T(X)\rangle]
$$
for any $\mu,\nu\in\mathcal{P}(\mathcal{X})$. Note that the integrability of $\Delta\text{RS}_T(X)$ is ensured by the fact that $X$ and $\hat{X}^\theta$ attain only values in the compact set $\mathcal{X}$ almost surely and the activation function is bounded by Assumption \ref{ass:assumactivation}. In particular, the problem becomes linear. Hence, using the linearity of the expectation and by choosing the Euclidean norm $\Vert \cdot \Vert := \Vert \cdot \Vert_{2}$ we obtain the final expression
\begin{equation}\label{eq:randsigW1}
   \text{RS-}W_1(\mu,\nu) = \Vert \mathbb{E}_{\mu}[\Delta\text{RS}_T(X)] - \mathbb{E}_{\nu}[\Delta\text{RS}_T(X)]\Vert_{2}.
\end{equation}
We use the (pseudo-)metric \eqref{eq:randsigW1} for generating synthetic trajectories of the process $X$ based on $\hat{X}^\theta$ defined in \eqref{eq:generatorSDE}. In the tradition of GAN models, we call the system \eqref{eq:generatorSDE} the \textit{generator} and the metric \eqref{eq:randsigW1} the \textit{discriminator}. Therefore, we consider the following generator-discriminator pair 
\begin{equation}\label{eq:gen_disc_pair_uncond}
    \textbf{Generator: } \hat{X}^\theta \hspace{1cm} \textbf{Discriminator: } \text{RS-}W_1(\mathbb{P}^\text{real},\mathbb{P}^\text{fake}_\theta)
\end{equation}
so that the corresponding training objective \eqref{eq:trainingproblem} can be written as
\begin{equation}\label{eq:rsigW1objective}
    \min_\theta \text{RS-}W_1(\mathbb{P}^{\text{real}}, \mathbb{P}^{\text{fake}}_{\theta}) = \min_\theta \Vert \mathbb{E}_{\mathbb{P}^\text{real}}[\Delta\text{RS}_T(X)] - \mathbb{E}_{\mathbb{P}^\text{fake}_\theta}[\Delta\text{RS}_T(X)]\Vert_{2}.
\end{equation}
It is worth noting that the training problem of traditional GAN models are formulated as a min-max problem. This leads to an adversarial training, which might cause problems in the optimal numerical implementation, since the problem might not converge to a solution as described in \cite{Daskalakis2017, Daskalakis2018}. By considering the RS-W$_1$ metric as the discriminator, we are able to formulate the training objective as a minimisation problem of a convex function, so that the training becomes non-adversarial. A similar approach using signature was pursued in \cite{Ni2021}, \cite{Liao2023}, \cite{Issa2023} and \cite{Lozano2023}.

\subsection{Numerical results}
In this section, we present out-of-sample results for the generator-discriminator pair \eqref{eq:gen_disc_pair_uncond}. In particular, we generate timeseries data based on both synthetic and real world training data. The generator proposed in \eqref{eq:generatorSDE} is compared to the long-short-term-memory (LSTM) model introduced in \cite{Hochreiter1997} for timeseries generation. Moreover, we compare the discriminator metric to the Sig-W$_1$ metric proposed in \cite{Ni2021}. It is worth noting that the Sig-W$_1$ discriminator was tested against other benchmark models as well. We consider the path signature of time, lead-lag and visibility augmented input paths to compute the Sig-$W_1$ metric. We refer to \cite{Ni2021} for more details. To measure the quality of generated paths we compare the corresponding discriminator metric, as well as metrics for the covariance and autocorrelation introduced in Appendix \ref{section:eval_metrics}. We also apply a Shapiro-Wilk test for normality as defined in \cite{Shapiro1965} in case the generated data should be normally distributed based on its training data. For all experiments we use 80\% of the available samples for the training set and 20\% for the test set. Moreover, we apply the Adam algorithm \cite{Kingma2019} for optimising \eqref{eq:generatorSDE}, where each gradient step is based on a batch of randomly selected training paths. Table \ref{table:uncond_hyperparameters} gives on overview of the hyperparameters used. For this entire section, we set $\rho_5^\theta$ in \eqref{eq:generatorSDE} equal to one. However, for longer time horizons and to generate paths with very small variance it proved to enhance the generator's performance to include $\rho_5^\theta$ in the set of trainable parameters. It is worth noting that also the dimension of $N$ of the randomised signature is a hyperparameter and tuned for the corresponding number of timesteps. For experiments with larger time horizons, the choice of $N$ must be increased as well. Heuristically, Sampling Scheme \ref{samplingscheme1} and \ref{samplingscheme2} indicate that the relationship of $N$ and $T$ should be chosen linearly.

\subsubsection{Brownian motion with drift}\label{subsection:experiment_BM_with_drift}
First, we generate trajectories of a discretised Brownian motion with drift $\tilde{W} = (\tilde{W}_t)_{t=1,\dots,T}$ defined as 
\begin{equation}\label{eq:bm_scheme}
\begin{cases}
    \tilde{W}_1 = 0\\
    \tilde{W}_t = \tilde{W}_{t-1} + \mu + \sigma Z_t,\; t=2,\dots,N
\end{cases}
\end{equation}
for a drift parameter $\mu\in\mathbb{R}$, volatility $\sigma>0$ and independent $Z_t\sim\mathcal{N}(0,1)$, $t=2,\dots,T$. In our experiment we study paths of both a Brownian motion with and without drift. The out-of-sample results are summarised in Table \ref{table:results_uncond_bm}.

We observe that for both drift parameters, our proposed generator-discriminator pair generates data with normal marginals due to the Shapiro-Wilk test. Further, based on out-of-sample data, both the covariance and autocorrelation metrics provide the lowest values for our model implying a more realistic dependence structure. It is also worth noting that the Sig-W$_1$ discriminator performs better combined with our generator than the LSTM model. 

\begin{table}[H]
\captionsetup{width=.8\linewidth}
\begin{center}\resizebox{0.8\textwidth}{!}{
\begin{tabular}{c||c|c|c|c} 
    & Train Metric & Cov. Metric & ACF Metric & SW Tests Passed \\[0.5ex]  
    \hline\hline 
    $\mu=0,\;\sigma=1$ & & & & \\[0.5ex]  
    \hline 
     RS-$W_1$ - LSTM & $1.75\times 10^{-1}$ & $2.61\times10^{-1}$ & $1.97\times10^{-1}$ & $6/9$ \\[0.5ex]  
     RS-$W_1$ - NeuralSDE & $\mathbf{1.31\times 10^{-1}}$ & $\mathbf{1.98\times 10^{-1}}$ & $\mathbf{6.64\times 10^{-2}}$ & $\mathbf{9/9}$  \\[0.5ex] 
     \hline
     Sig-$W_1$ - LSTM & $\mathbf{8.15\times10^{-1}}$ & $9.14\times10^{-1}$ & $4.67\times10^{-1}$ & $5/9$ \\[0.5ex]  
     Sig-$W_1$ - NeuralSDE & $9.78\times 10^{-1}$  & $8.76\times 10^{-1} $  & $4.18\times 10{-1}$  & $7/9$ \\[0.5ex]
     \hline\hline 
    $\mu=1,\;\sigma=1$ & & & & \\[0.5ex]  
    \hline 
     RS-$W_1$ - LSTM & $4.41\times 10^{-1}$ & $9.32\times 10^{-1}$ & $3.45\times 10^{-1}$ & $\mathbf{9/9}$  \\[0.5ex]  
     RS-$W_1$ - NeuralSDE & $\mathbf{3.71\times 10^{-1}}$ & $\mathbf{7.94\times10^{-1}}$ & $\mathbf{2.31\times10^{-1}}$ & $\mathbf{9/9}$ \\[0.5ex]  
     \hline
     Sig-$W_1$ - LSTM & $9.81\times10^{-1}$ & $8.55\times 10^{-1}$ & $2.63\times 10^{-1}$ & $8/9$ \\[0.5ex]  
     Sig-$W_1$ - NeuralSDE & $\mathbf{8.23\times10^{-1}}$ & $8.29\times 10^{-1}$ & $3.51\times 10^{-1}$ & $8/9$ \\[0.5ex]  
\end{tabular}}
\end{center}
\caption{Out-of-sample test results for different generator-discriminator pairs generating paths of a Brownian motion.}
\label{table:results_uncond_bm}
\end{table}

The following plot displays $50$ trajectories of the test set and generated paths of the trained generator-discriminator pair \eqref{eq:gen_disc_pair_uncond}.

\begin{center}
    \begin{figure}[H]
        \centering
        \captionsetup{width=.8\linewidth}
        \includegraphics[width=.7\textwidth]{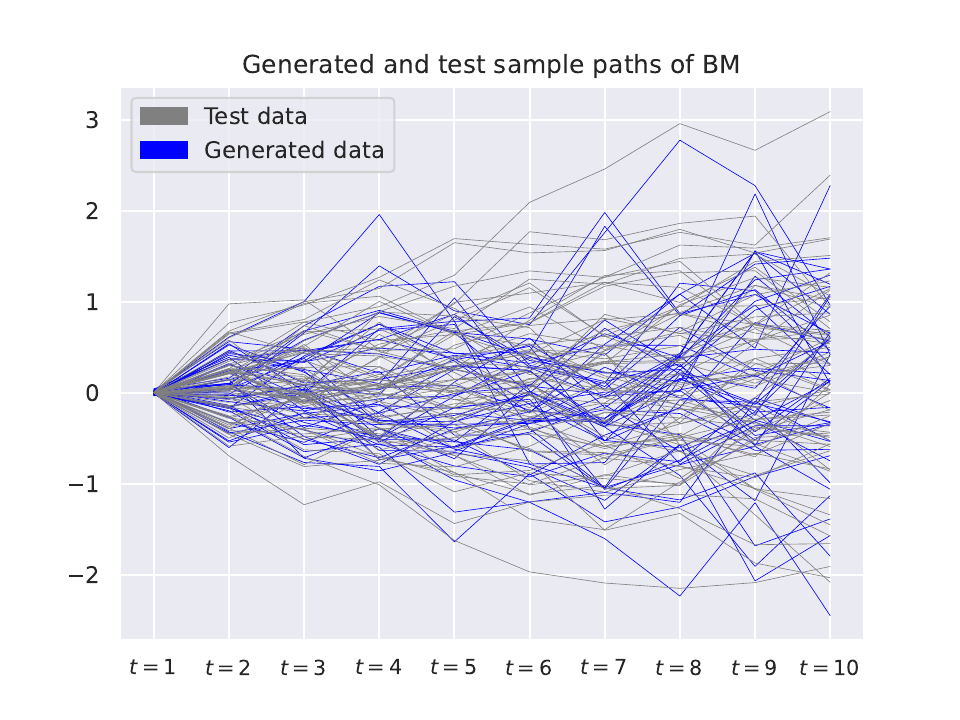}
        \caption{$50$ randomly selected trajectories from both generated and real out-of-sample samples of a Brownian motion with drift $\mu=0$ and variance $\sigma=1$.}
        \label{fig:uncond_BM_paths}
    \end{figure}
\end{center}
For the training of the generator the training set consists of $8000$ and the test set of $2000$ sample paths, respectively. Moreover, the numerical optimisation performed $2500$ gradient steps.

\subsubsection{Autoregressive process}
In this section, we study the performance of the generator-discriminator pair \eqref{eq:gen_disc_pair_uncond} generating trajectories of an autoregressive process. In particular, recall that for some $p\in\mathbb{N}$ a process $X=(X_t)_{t\in\mathbb{R}}$ is an AR$(p)$-process if it is of the form
\begin{equation}\label{eq:AR_process}
    X_t = \sum_{i=1}^{p} \varphi_i X_{t-i} + Z_t,
\end{equation}
where $\varphi_1,\dots,\varphi_p\in\mathbb{R}$ are constants controlling the influence of past observations of the current one and $Z_1,\dots,Z_T\sim\mathcal{N}(0,\sigma^2)$ are i.i.d for some $\sigma>0$. The stationarity of the process $X$ can be characterised by the roots of the autoregressive polynomial 
$$
    \Phi(x) := 1 - \varphi_1 x - \dots - \varphi_p x,\;x\in\mathbb{R}.
$$
In fact, the process $X$ is stationary if and only if the roots of the polynomial $\Phi$ lie outside the unit disc. For example, for $p=1$ the process $X$ is stationary if and only if $\vert \varphi_1\vert < 1$. The training and test set have the same size as in Section \ref{subsection:experiment_BM_with_drift}. Below we present out-of-sample results for an AR($1$)-process for different values of $\varphi_1$. 

\begin{table}[H]
\captionsetup{width=.7\linewidth}
\begin{center}\resizebox{0.7\textwidth}{!}{
\begin{tabular}{c||c|c|c} 
    & Train Metric & Cov. Metric & ACF Metric\\[0.5ex]  
    \hline\hline 
    $p=1,\;\varphi_1 = 0.1$ & & & \\[0.5ex]  
    \hline 
     RS-$W_1$ - LSTM & $8.76\times10^{-2}$ & $2.98\times10^0$ & $7.17\times10^{-1}$  \\[0.5ex]  
     RS-$W_1$ - NeuralSDE & $\mathbf{8.18\times10^{-2}}$ & $\mathbf{2.78\times10^0}$ & $\mathbf{2.11\times10^{-1}}$  \\[0.5ex] 
     \hline
     Sig-$W_1$ - LSTM & $9.58\times10^{-1}$ & $4.83\times10^{0}$ & $6.04\times10^{-1}$  \\[0.5ex]  
     Sig-$W_1$ - NeuralSDE & $\mathbf{9.34\times10^{-1}}$ & $3.61\times10^0$ & $3.48\times10^{-1}$ \\[0.5ex]
     \hline\hline 
    $p=1,\;\varphi_1 = 0.9$ & & & \\[0.5ex]  
    \hline 
     RS-$W_1$ - LSTM & $2.53\times10^{-1}$ & $\mathbf{2.98\times10^0}$ & $3.51\times10^{-1}$ \\[0.5ex]  
     RS-$W_1$ - NeuralSDE & $\mathbf{1.32\times10^{-1}}$ & $3.37\times10^0$ & $\mathbf{3.37\times10^{-1}}$ \\[0.5ex]  
     \hline
     Sig-$W_1$ - LSTM & $8.84\times10^{-1}$ & $7.91\times10^{0}$ & $5.30\times10^{-1}$ \\[0.5ex]  
     Sig-$W_1$ - NeuralSDE & $\mathbf{7.58\times10^{-1}}$ & $5.32\times10^0$ & $5.17\times10^{-1}$ \\[0.5ex]  
     \hline\hline 
     $p=1,\;\varphi_1 = -0.1$ & & & \\[0.5ex]  
    \hline 
     RS-$W_1$ - LSTM & $1.23\times10^{-1}$ & $3.02\times10^0$ & $8.41\times10^{-1}$  \\[0.5ex]  
     RS-$W_1$ - NeuralSDE & $\mathbf{9.11\times10^{-2}}$ & $\mathbf{1.26\times10^0}$ & $\mathbf{1.24\times10^{-1}}$ \\[0.5ex] 
     \hline
     Sig-$W_1$ - LSTM & $9.64\times10^{-1}$ & $4.71\times10^0$ & $5.06\times10^{-1}$ \\[0.5ex]  
     Sig-$W_1$ - NeuralSDE & $\mathbf{9.41\times10^{-1}}$ & $3.00\times10^0$ & $4.37\times10^{-1}$ \\[0.5ex]
     \hline\hline 
     $p=1,\;\varphi_1 = -0.9$ & & & \\[0.5ex]  
    \hline 
     RS-$W_1$ - LSTM & $2.11\times10^{-1}$ & $7.22\times10^0$ & $9.87\times10^{-1}$ \\[0.5ex]  
     RS-$W_1$ - NeuralSDE & $1\mathbf{.68\times10^{-1}}$ & $\mathbf{4.25\times10^0}$ & $\mathbf{9.42\times10^{-1}}$ \\[0.5ex]  
     \hline
     Sig-$W_1$ - LSTM & $1.32\times10^0$ & $7.50\times10^0$ & $1.12\times10^0$ \\[0.5ex]  
     Sig-$W_1$ - NeuralSDE & $\mathbf{9.89\times10^{-1}}$ & $7.38\times10^0$ & $9.85\times10^{-1}$ \\[0.5ex]  
\end{tabular}}
\end{center}
\caption{Out-of-sample test results for different generator-discriminator pairs generating paths of an AR$(1)$ process with different parameter values $\varphi_1$.}
\end{table}

Besides the covariance metric in the case $\varphi_1=0.9$, our proposed model yields the lowest test metrics in all cases compared to the other generator-discriminator pairs. For $\varphi_1=0.9$ the result is not much worse than for the RS-W$_1$ combination and still better than the models using the Sig-W$_1$ metric as the discriminator. Nevertheless, the latter performs better with our proposed generator \eqref{eq:generatorSDE} than the LSTM for all values of $\varphi_1$. We refer to Figure \ref{fig:kde_ar} in the Appendix for plots of kernel density estimated comparing the generated and test data for two marginals.

\subsubsection{S\&P 500 return data}\label{subsection:sp500}
Next, we apply the generator model to real-world financial data. In this section, we consider daily log-returns of the S\&P 500 index based on its closing prices between 01/01/2005 and 31/12/2023. Note that it is often assumed that log-returns are (almost) stationary over time. Therefore, we apply a rolling window to ``cut" the time series into 4316 sample paths of length $10$, which is justified by the stationary behaviour of return data. The training set consists of $80$\% of the samples, while the rest is used for out-of-sample testing. In the following, we present the out-of-sample results on the trained generator for different dimensions $N$ of the randomised signature as defined in \eqref{eq:reservoirsystem}.

\begin{table}[H]
\captionsetup{width=.8\linewidth}
\begin{center}\resizebox{0.8\textwidth}{!}{
\begin{tabular}{c||c|c|c} 
    & Train Metric & Cov. Metric & ACF Metric\\[0.5ex]  
    \hline\hline 
     RS-$W_1$ - LSTM $(N=50)$ & $2.01\times10^{-1}$ & $1.03\times10^0$ & $6.13\times10^{-1}$ \\[0.5ex]  
     RS-$W_1$ - NeuralSDE $(N=50)$ & $\mathbf{1.84\times10^{-1}}$ & $9.84\times10^{-1}$ & $3.32\times10^{-1}$  \\[0.5ex] 
     \hline
     RS-$W_1$ - LSTM $(N=80)$ & $1.99\times10^{-1}$ & $9.12\times10^{-1}$ & $3.82\times10^{-1}$ \\[0.5ex]  
     RS-$W_1$ - NeuralSDE $(N=80)$ & $\mathbf{1.14\times10^{-1}}$ & $\mathbf{8.40\times10^{-1}}$ & $\mathbf{1.16\times10^{-1}}$ \\[0.5ex] 
     \hline
     RS-$W_1$ - LSTM $(N=100)$ & $2.85\times10^{-1}$ & $9.14\times10^{-1}$ & $4.98\times10^{-1}$ \\[0.5ex] 
     RS-$W_1$ - NeuralSDE $(N=100)$ & $\mathbf{1.93 \times 10^{-1}}$ & $9.02\times10^{-1}$ & $4.32\times10^{-1}$ \\[0.5ex] 
     \hline
     Sig-$W_1$ - LSTM & $4.99\times10^0$ & $9.22\times10^0$ & $8.53\times10^{-1}$  \\[0.5ex]  
     Sig-$W_1$ - NeuralSDE & $\mathbf{3.14\times10^0}$ & $9.98\times10^{-1}$ & $3.91\times10^{-1}$ \\[0.5ex]
\end{tabular}}
\end{center}
\caption{Out-of-sample test results for different generator-discriminator pairs generating paths of S\&P 500 log-returns.}
\end{table}

We see that for all generative model containing the RSig-W$_1$ discriminator the choice $N=80$ yields the best result for all metrics. This observation was one reason for choosing this hyperparameter for generating paths of the Brownian motion and the AR(1)-process. In fact for all values of $N$, our proposed model yields lower test metric values than the models using the Sig-W$_1$ discriminator. Lastly, as before we see that the latter performs better paired with the generative model \eqref{eq:generatorSDE} than the LSTM. The plot in Figure \ref{fig:kde_sp500} in the Appendix displays the kernel density estimates for generated and test data. 

\subsubsection{FOREX return data}
Similar to Section \ref{subsection:sp500}, we compute the log-returns of FOREX EUR-USD on the closing rates between 01/11/2023 and 31/01/2024. The used data can be downloaded from \url{https://forexsb.com/historical-forex-data}. Again, based on a stationary assumption, we collect minute-tick data and apply a rolling window the receive 5831 sample paths of length $10$. The training and test set are constructed in the same manner as in Section \ref{subsection:sp500}. The following table summarises out-of-sample results on the trained generator for different dimensions $N$ of the randomised signature as defined in \eqref{eq:reservoirsystem}.

\begin{table}[H]
\captionsetup{width=.8\linewidth}
\begin{center}\resizebox{0.8\textwidth}{!}{
\begin{tabular}{c||c|c|c} 
    & Train Metric & Cov. Metric & ACF Metric\\[0.5ex]  
    \hline\hline 
     RS-$W_1$ - LSTM $(N=50)$ & $3.21\times10^{-1}$ & $2.85\times10^0$ & $8.63\times10^{-1}$ \\[0.5ex]  
     RS-$W_1$ - NeuralSDE $(N=50)$ & $\mathbf{2.53\times10^{-1}}$ & $4.53\times10^0$ & $8.15\times10^{-2}$ \\[0.5ex] 
     \hline
     RS-$W_1$ - LSTM $(N=80)$ & $\mathbf{9.01\times10^{-2}}$ & $3.21\times10^{0}$ & $8.21\times10^{-1}$  \\[0.5ex]  
     RS-$W_1$ - NeuralSDE $(N=80)$ & $1.01\times10^{-1}$ & $\mathbf{2.67\times10^{0}}$ & $\mathbf{6.19\times10^{-2}}$  \\[0.5ex] 
     \hline
     RS-$W_1$ - LSTM $(N=100)$ & $ 1.91\times10^{-1}$ & $3.42\times10^{0}$ & $4.03\times10{-1}$ \\[0.5ex]  
     RS-$W_1$ - NeuralSDE $(N=100)$ & $\mathbf{1.76\times10^{-1}}$ & $ 3.39\times10^0$ & $2.25\times10^{-1}$ \\[0.5ex] 
     \hline
     Sig-$W_1$ - LSTM & $4.36\times10^0$ & $9.23\times10^0$ & $2.98\times10^{-1}$ \\[0.5ex]  
     Sig-$W_1$ - NeuralSDE & $\mathbf{2.85\times10^0}$ & $8.52\times10^0$ & $2.86\times10^{-1}$ \\[0.5ex]
\end{tabular}}
\end{center}
\caption{Out-of-sample test results for different generator-discriminator pairs generating paths of FOREX EUR/USD log-returns.}
\label{tab:results_forex_uncond}
\end{table}

Similar to the results for the S\&P 500 log-returns, Table \ref{tab:results_forex_uncond} indicates that our model yields the best results for both test metrics if the dimension of randomised signature is chosen as $N=80$. Moreover, the model configurations incorporating the RS-W$_1$ metric leads to more realistic results than the Sig-W$_1$ discriminator. The combination Sig-$W_1$ and the generator \eqref{eq:generatorSDE} has lower test metrics compared to the LSTM generator. The kernel density estimates can be found in Figure \ref{fig:kde_forex} in the Appendix. 

\section{Conditional time series generation using the randomised signature}\label{section4}
In this section, we study the problem of generating future trajectories of a path given some information of its past. Let $K\subseteq \mathbb{R}^d$ be a compact subset. For $p,q\in\mathbb{N}$ we denote by $\mathcal{Y}:= K\times\dots\times K\subseteq(\mathbb{R}^d)^p$ and $\mathcal{Z}:= K\times\dots\times K\subseteq(\mathbb{R}^d)^q$. We consider a $\mathcal{Y}\times\mathcal{Z}$-valued adapted stochastic process $X=(X_t)_{t=1,\dots,p,\dots,p+q}$ defined on $(\Omega,\mathcal{F},\mathbb{F},\mathbb{P})$. Moreover, we define 
\begin{equation*}
    X^{\text{past},p} := (X_1,\dots,X_p),\; X^{\text{future},q} := (X_{p+1},\dots,X_{p+q}),
\end{equation*}
hence it holds $X=(X^{\text{past},p},X^{\text{future},q})$. We denote by $\mathbb{P}^r_c(x)$ the conditional distribution of $X^{\text{future},q}$ given the information that $X^{\text{past},p} = x$ for some $x\in\mathcal{Y}$. Similar to the unconditional case in Section \ref{Section3}, our aim is to generate trajectories of a $K^q$-valued random process $\hat{X}^{q,\theta} = (\hat{X}^{q,\theta}_t)_{t=1,\dots,T}$ with conditional distribution $\mathbb{P}^{f,\theta}_c(x)$ on the information $X^{\text{past},p} = x$ such that $d(\mathbb{P}^r_c(x), \mathbb{P}^{f,\theta}_c(x))$ is small, where $d$ is a metric on the space of probability measures on $\mathcal{Z}$ specified later.

Consider the $n$-dimensional $\mathbb{F}$-Brownian motion $W=(W^1_t,\dots,W^n_t)_{t\in[0,T]}$ and the random vector $V\sim\mathcal{N}_m(0,\text{Id})$ defined in Section \ref{Section3}. For some $D\in\mathbb{N}$ and given $x\in\mathcal{Y}$, we consider the $D$-dimensional stochastic process $\tilde R = (\tilde{R}_t)_{t=1,\dots,q}$ and the $d$-dimensional process $\bar{X}^{q,\theta}_t=(\bar{X}^{q,\theta}_t)_{t=1,\dots,q}$ defined as follows
\begin{equation}\label{eq:condgeneratorSDE}
    \begin{cases}
        \tilde{R}_1 = \tilde{\Psi}^\theta(V, \Delta\text{RS}_p(x)),\\[2mm]
        \tilde{R}_t = \tilde{R}_{t-1} + \sigma(\tilde{\rho}_1^\theta \tilde{B}_1 \tilde{R}_{t-1} + \tilde{\rho}_2^\theta \tilde{\lambda}_1) + \sum_{i=1}^n \sigma(\tilde{\rho}_3^\theta \tilde{B}^i_2 \tilde{R}_{t-1} + \tilde{\rho}_4^\theta \tilde{\lambda}^i_2)\tilde{\rho}_5^\theta(W^i_{t} - W^i_{t-1}),\\[2mm]
        \bar{X}^{q,\theta}_t = \tilde{A}^\theta_t \tilde{R}_t + \tilde{\beta}^\theta_t, \; t = 1,\dots, q,
    \end{cases}
\end{equation}
for a neural network $\tilde{\Psi}^\theta:\mathbb{R}^m \times \mathbb{R}^N \to \mathbb{R}^D$, random matrices $\tilde{B}_1, \tilde{B}^1_2,\dots, \tilde{B}^n_2 \in\mathbb{R}^{D\times D}$, $\tilde{\lambda}_1, \tilde{\lambda}^1_2,\dots,\tilde{\lambda}^n_2 \in\mathbb{R}^{D}$, parameters $\tilde{\rho}_1^\theta,\dots, \tilde{\rho}_5^\theta\in\mathbb{R}$, $\tilde{A}^\theta_1,\dots,\tilde{A}^\theta_q\in\mathbb{R}^{d\times D}$, $\tilde{\beta}^\theta_1,\dots,\tilde{\beta}^\theta_q\in\mathbb{R}^d$ and some activation function $\sigma:\mathbb{R}\to\mathbb{R}$ applied component-wise. The entries of the random matrices are i.i.d.\hspace{-1mm} with some absolutely continuous distribution $\mu_4$ with full support and independent of $X$, $W$ and $V$. As before, the superscript $\theta$ denotes the trainable weights of the corresponding neural network or matrix. Finally, the process $\hat{X}^{q,\theta}$ is defined by the projection of $\bar{X}^{q,\theta}$ on a compact set. In particular, for some $r>0$ consider the projection defined for any $x\in\mathbb{R}^d$ as
\begin{equation*}
    \pi_r(x) =
    \begin{cases}
        x, \text{ if } \Vert x\Vert <r,\\
        r\frac{x}{\Vert x\Vert}, \text{ if } \Vert x\Vert\geq r.
    \end{cases}
\end{equation*}
Similar to \eqref{eq:generatorSDE}, the process $\tilde{R}$ is a variant of randomised signature of $(W_t - W_{t-1})_{t=1,\dots,q}$ with initial condition chosen to incorporate the available past information $x$. For any $t=1,\dots,q$ we define $\hat{X}^{q,\theta}_t := \pi_r(\bar{X}^{q,\theta}_t)$, where the constant $r$ is chosen such that $\hat{X}^{q,\theta}\in \mathcal{Z}$. The conditional version of the problem \eqref{eq:trainingproblem} can be written 
\begin{equation}\label{eq:trainingproblem_cond}
    \min_{\theta} d_\text{path}\left(\mathbb{P}^{r}_{c}(x), \mathbb{P}^{f,\theta}_{c}(x)\right)
\end{equation}
for an input path $x\in\mathcal{Y}$. As described before, the entries of the matrices $\tilde{B}_1, \tilde{B}^1_2,\dots, \tilde{B}^n_2$ and vectors $\tilde{\lambda}_1, \tilde{\lambda}^1_2,\dots,\tilde{\lambda}^n_2$ are once sampled from an absolutely continuous distribution and then fixed for the training procedure.

\subsection{The conditional randomised signature Wasserstein-1 (C-RS-$W_1$) metric}
In this section, we introduce a conditional version of the RS-$W_1$ metric introduced in Section \ref{Section3}. Consider the conditional distributions of $\mathbb{P}^{r}_{c}(x)$ and $\mathbb{P}^{f,\theta}_{c}(x)$, as defined above. Based on the construction of the unconditional metric, we define the conditional randomised signature Wasserstein-1 metric as 
$$
    \text{C-RS-}W_1(\mathbb{P}^{r}_{c}(x),\mathbb{P}^{f,\theta}_{c}(x)) := \left\Vert \mathbb{E}_{\mathbb{P}^{r}_{c}(x)}[\Delta\text{RS}_q(X^{\text{future},q})] - \mathbb{E}_{\mathbb{P}^{f,\theta}_{c}(x)}[\Delta\text{RS}_q(X^{\text{future},q})] \right\Vert_2.
$$
Note that integrability of $\Delta\text{RS}_T(X^{\text{future},q})$ is ensured by its definition and the fact that $X^{\text{future},q}$ and $\hat{X}^{q,\theta}$ only attain values on a compact domain and that $\sigma$ is bounded by Assumption \ref{ass:assumactivation}. For an input path $x=(x_1,\dots,x_p)\in\mathcal{Y}$, we obtain the following generator-discriminator pair
\begin{equation}\label{eq:gen_discriminator_pari_cond}
    \textbf{Generator: } \hat{X}^{q,\theta} = (\hat{X}^{q,\theta}_1,\dots,\hat{X}^{q,\theta}_q)\hspace{1cm} \textbf{Discriminator: } \text{C-RS-}W_1(\mathbb{P}^{r}_{c}(x), \mathbb{P}^{f,\theta}_{c}(x))
\end{equation}
using the system \eqref{eq:condgeneratorSDE} as the generator and the C-RS-$W_1$ metric as the discriminator so that the training objective \eqref{eq:trainingproblem_cond} can be written as
\begin{equation}\label{eq:crsigW1objective}
    \min_\theta \left\Vert \mathbb{E}_{\mathbb{P}^{r}_{c}(x)}[\Delta\text{RS}_q(X^{\text{future},q})] - \mathbb{E}_{\mathbb{P}^{f,\theta}_{c}(x)}[\Delta\text{RS}_q(X^{\text{future},q})] \right\Vert_2.
\end{equation}
We devote the following section to describe how the learning problem \eqref{eq:crsigW1objective} is implemented numerically, since in the conditional setting simulating the conditional expectation under $\mathbb{P}^r_c(x)$ for some $x\in\mathcal{Y}$ is not straightforward. A similar approach using path signatures was pursued in \cite{Liao2023}.

\subsection{Practical implementation}
As for the unconditional case, we can estimate $\mathbb{E}_{\mathbb{P}^{f,\theta}_{c}(x)}[\Delta\text{RS}_q(X^{\text{future},q})]$ by a Monte-Carlo sum. For a fixed input path $x$ we generate samples of the future path using \eqref{eq:condgeneratorSDE}, compute their corresponding randomised signatures and take their average. However, this is not possible for the conditional expectation $\mathbb{E}_{\mathbb{P}^{r}_{c}(x)}[\Delta\text{RS}_q(X^{\text{future},q})]$. Note that for a fixed past we have only one future trajectory for the real data so that a Monte-Carlo approach is not applicable. In fact, in typical financial applications we observe only one long datastream/timeseries $x=(x_1,\dots,x_T)\in K^T$ for some $T\in\mathbb{N}$, which can be interpreted as the realisation of some stochastic process. Normally it holds that $T$ is significantly bigger than $p$ and $q$. Therefore, in practice we apply a rolling window to the datastream to obtain the samples 
\begin{equation*}
    x^{\text{past},p}_i := (x_{i-p+1},\dots,x_i)\in\mathcal{Y},\; x^{\text{future},q}_i := (x_{i+1},\dots,x_{i+q})\in\mathcal{Z}
\end{equation*}
for $i=p,\dots,T-q$. Assuming that the observations are stationary in time, we interpret $(x^{\text{past},p}_i)_{i=p}^{T-q}$ as samples of $X^{\text{past},p}$ and $(x^{\text{future},q}_i)_{i=p}^{T-q}$ of $X^{\text{future},q}$ respectively. As described above, for every sample path $x^{\text{past},p}_i$ there is only one corresponding future sample path $x^{\text{future},q}_i$ so that Monte-Carlo techniques fail. Therefore, we use a method proposed in \cite{Levin2013}, which is based on the Doob-Dynkin Lemma: we assume that there exists a measurable function $l:\mathcal{Y}\to\mathbb{R}^N$ such that
\begin{equation}\label{eq:Doob_Dynkin_function}
    l(x^{\text{past},p}_i) = \mathbb{E}_{\mathbb{P}^r_c(x^{\text{past},p}_i)}\left[\Delta\text{RS}_q(X^{\text{future},q})\right]
\end{equation}
for every $i=p,\dots,T-q$. Assuming that $l$ is continuous, it is intuitive to approximate \eqref{eq:Doob_Dynkin_function} by a linear function on $\Delta\text{RS}_q(x^{\text{past},p}_i)$, which is motivated by the universality results presented in Section \ref{Section2}. In particular, approximating $l$ reduces to a linear regression problem. In fact, we assume that for every $i=p,\dots,T-q$ it holds
\begin{equation}\label{eq:linear_regression_problem}
    \Delta\text{RS}_q(x^{\text{future},q}_i) = \alpha  + \beta \Delta\text{RS}_p(x^{\text{past},p}_i) + \varepsilon_i
\end{equation}
for $\alpha\in\mathbb{R}^N$, $\beta\in\mathbb{R}^{N\times N}$ and error terms $\varepsilon_i\in\mathbb{R}^N$. We denote by $\hat{\alpha}\in\mathbb{R}^N,\hat{\beta}\in\mathbb{R}^{N\times N}$ the solution to the problem 
\begin{equation}\label{eq:regression_problem}
    \min_{\alpha\in\mathbb{R}^N,\beta\in\mathbb{R}^{N\times N}}\sum_{i=p}^{T-q} \left\Vert \alpha + \beta \Delta\text{RS}_p(x^{\text{past},p}_i) - \Delta\text{RS}_q(x^{\text{future},q}_i)\right\Vert^2_2.
\end{equation}
Note that the regression problem \eqref{eq:regression_problem} can be solved prior to the actual training problem \eqref{eq:trainingproblem_cond}, which improves the computation time of the problem in general. Once the solution $(\hat{\alpha},\hat{\beta})$ is found, \eqref{eq:crsigW1objective} becomes the following supervised learning problem 
\begin{equation*}
    \min_\theta \left\Vert \hat{\alpha} + \hat{\beta} \Delta\text{RS}_p(x) - \mathbb{E}_{\mathbb{P}^{f,\theta}_c(x)}[\Delta\text{RS}_q(x)]\right\Vert_2.
\end{equation*}
The entire implementation of the training procedure is described in Algorithm \ref{alg:conditional_gen} which can be found in Appendix \ref{section:pseudocode_cond_gen}.

\subsection{Empirical results}
In this section, we present out-of-sample results for the generator-discriminator pair \eqref{eq:gen_discriminator_pari_cond}. In particular, we generate future trajectories of a timeseries data conditional on a path's past for both synthetic and real-world training data. For the generator proposed in \eqref{eq:condgeneratorSDE} we compare the result for using the C-RS-$W_1$ metric as the discriminator to the conditional version of Sig-W$_1$ metric proposed in \cite{Liao2023}. As in Section \ref{Section3}, we state the value of the corresponding discriminator metric and the autocorrelation distance proposed in \eqref{eq:acf_diff} based on out-of-sample test data. Note that the covariance metric cannot be estimated using \eqref{eq:cov_diff} since only one trajectory of the real data is available so that a Monte-Carlo approach fails at this point. For an overview of the hyperparameters used during the training we refer to Table \ref{table:cond_hyperparameters} in the Appendix. Similar to Section \ref{Section3}, we set the parameter $\tilde{\rho}_5^\theta$ in \eqref{eq:condgeneratorSDE} equal to one for the following experiments. The different model configurations are trained with $2500$ gradient steps or until no further improvement was achieved.

\subsubsection{Brownian motion}\label{subsection:conditional_BM_generation}
Similar to Section \ref{subsection:experiment_BM_with_drift}, we train the generator model in \eqref{eq:condgeneratorSDE} to generate sample paths of a Brownian motion conditional on some information about its past. We use \eqref{eq:bm_scheme} to generate training and test paths of length $15$. For each path the first $5$ datapoints are considered as the past, while the remaining $10$ are supposed to be the corresponding future. The training set consists of $8000$ samples and the test set of $2000$. The following table summarises the result for Brownian motion paths with and without drift. Moreover, we use the Shapiro-Wilk test to study if the marginals of the generated paths follow a normal distribution with high statistical confidence. 

\begin{table}[H]
\captionsetup{width=.8\linewidth}
\begin{center}
\resizebox{0.8\textwidth}{!}{
\begin{tabular}{c||c|c|c} 
    & Train Metric & ACF Metric & SW Tests Passed \\[0.5ex]  
    \hline\hline 
    $\mu=0,\;\sigma=1$ & & & \\[0.5ex]  
    \hline 
     C-RS-$W_1$ - NeuralSDE & $1.46\times10^0$  & $\mathbf{3.13\times10^{-2}}$  & $\mathbf{9/10}$  \\[0.5ex] 
     \hline
     C-Sig-$W_1$ - NeuralSDE & $2.28\times10^0$  & $6.25\times10^{-1}$  & $8/10$  \\[0.5ex]
     \hline\hline 
    $\mu=1,\;\sigma=1$ & & & \\[0.5ex]  
    \hline 
     C-RS-$W_1$ - NeuralSDE & $1.71\times10^0$  & $\mathbf{2.68\times10^{-1}}$  & $\mathbf{8/10}$  \\[0.5ex]  
     \hline
     C-Sig-$W_1$ - NeuralSDE & $2.42\times10^0$ & $3.25\times10^{-1}$ & $\mathbf{8/10}$ \\[0.5ex]  
\end{tabular}}
\caption{Out-of-sample results for different generator-discriminator pairs for conditional generation of Brownian motion.}
\label{tab:results_cond_BM}
\end{center}
\end{table}
The results indicate that the model configuration incorporating the C-RS-$W_1$ metric yields the best results for both the autocorrelation metric and the number of normality tests that could not reject the null hypothesis of a normal distribution. Hence, it is indicated that the C-RS-$W_1$ leads to more realistic paths especially in terms of temporal dependence with or without drift. The following plot displays trajectories generated by our proposed model based on the past of a single trajectory. The latter is highlighted in blue together with its actual future behaviour.

\begin{center}
    \begin{figure}[H]
        \centering
        \captionsetup{width=.8\linewidth}
        \includegraphics[width=.6\textwidth]{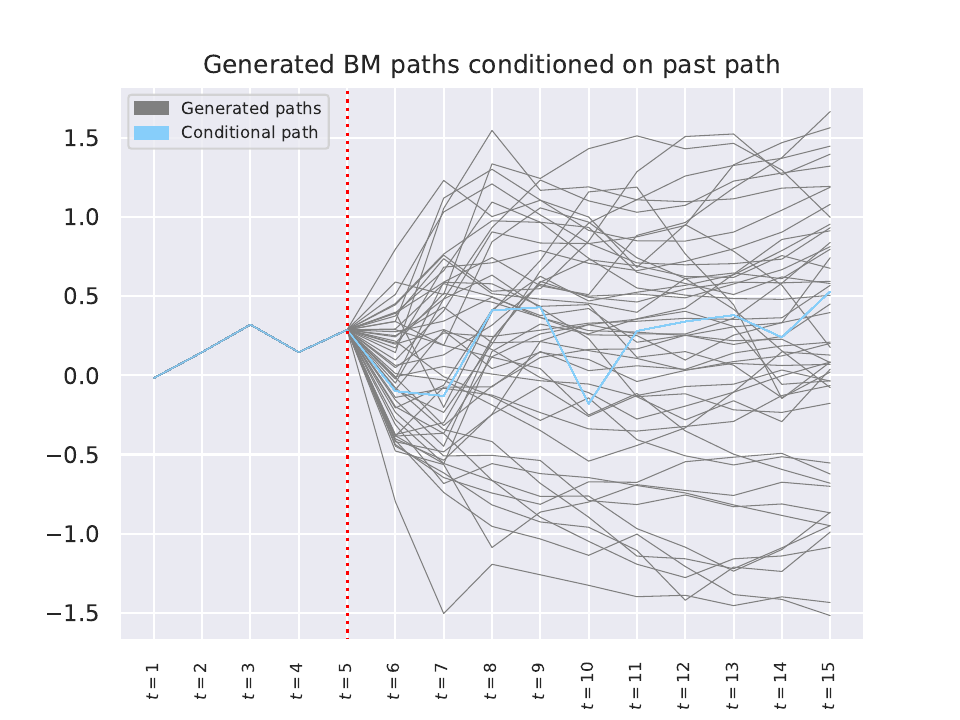}
        \caption{$50$ randomly selected generated future trajectories based on a past path of a Brownian motion with drift $\mu=0$ and variance $\sigma=1$.}
        \label{fig:cond_BM_paths}
    \end{figure}
\end{center}

\subsubsection{S\&P 500 return data}
We generate sample paths of future log-returns of the S\&P 500 index based on past returns. In particular, given the last 5 daily log-returns, we generate possible future scenarios for the coming $10$ days. We use historical data from 01/01/2005 to 31/12/2023 so that the training and test sets combined consist of $3415$ samples produced by a rolling window, where $80$\% are for training and $20$\% for out-of-sample testing. The following table presents out-of-sample results for three different dimensions for the randomised signature in C-RS-$W_1$ and the Sig-$W_1$ metric.

\begin{table}[H]
\captionsetup{width=.8\linewidth}
\begin{center}
\begin{tabular}{c||c|c|c} 
    & Train Metric & ACF Metric \\[0.5ex]  
    \hline\hline 
     C-RS-$W_1$ - NeuralSDE $(N=50)$ & $\mathbf{1.13\times10^0}$  & $\mathbf{8.96\times10^{-1}}$  \\[0.5ex]  
      \hline
     C-RS-$W_1$ - NeuralSDE $(N=80)$ & $1.42\times10^0$  & $9.30\times10^{-1}$  \\[0.5ex]  
      \hline
     C-RS-$W_1$ - NeuralSDE $(N=100)$ & $1.92\times10^0$  & $9.37\times10^{-1}$  \\[0.5ex] 
      \hline
     C-Sig-$W_1$ - NeuralSDE & $9.41\times10^0$  & $1.14\times10^0$  \\[0.5ex]
\end{tabular}
\caption{Out-of-sample results for different generator-discriminator pairs for conditional generation of S\&P 500 log-returns.}
\label{tab:results_cond_SP500ra}
\end{center}
\end{table}
It is indicated that the C-RS-$W_1$ leads to the lowest ACF distances for all three different dimensions $N$. Different to the results before, the best metric is obtained for $N=50$. However, the temporal dependence of the generated paths seems to be very similar for all three cases. 

\subsubsection{FOREX return data}
As for the S\&P 500 index, we also generate future log-return paths for the FOREX rate of Euro and US Dollar for the time period from 01/11/2005 to 31/12/2023. The observed market data yields $5826$ samples, of which $80$\% are used for the training and $20$\% for testing. The following table displays the corresponding training metric applied to the generated data and out-of-sample data as well as the ACF distance. For the C-RS-$W_1$ discriminator we consider three different dimensions $N$.
\begin{table}[H]
\captionsetup{width=.8\linewidth}
\begin{center}
\begin{tabular}{c||c|c|c} 
    & Train Metric & ACF Metric \\[0.5ex]  
    \hline\hline 
     C-RS-$W_1$ - NeuralSDE $(N=50)$ & $5.51\times10^{-1}$ & $9.81\times10^{-1}$ \\[0.5ex] 
      \hline
     C-RS-$W_1$ - NeuralSDE $(N=80)$ & $\mathbf{4.98\times10^{-1}}$ & $\mathbf{7.65\times10^{-1}}$ \\[0.5ex]  
      \hline
     C-RS-$W_1$ - NeuralSDE $(N=100)$ & $7.70\times10^{-1}$ & $1.05\times10^0$ \\[0.5ex]  
      \hline
     C-Sig-$W_1$ - NeuralSDE & $4.41\times10^0$  & $1.17\times10^0$ \\[0.5ex]
\end{tabular}
\caption{Out-of-sample results for different generator-discriminator pairs for conditional generation of FOREX EUR/USD log-returns.}
\label{tab:results_cond_FOREX}
\end{center}
\end{table}

We observe that the best C-RS-$W_1$ metric value is obtained for the configuration with dimension $N=80$, which also leads to the lowest ACF distance. Together with hyperparameter tuning for the experiment for Brownian motion, this observation led to the dimension choice for the results in Section \ref{subsection:conditional_BM_generation}. In all three cases C-RS-$W_1$ seems to yield a more realistic temporal dependence structure than Sig-$W_1$.


\acks{Niklas Walter gratefully acknowledges the financial support of the "Verein zur F\"orderung der Versicherungswissenschaft e.V.".}


\newpage

\appendix
\section{Overview of the training hyperparameters}
The following tables present the hyperparameters used for the training procedures.

\begin{table}[H]
\begin{center}
\begin{tabular}{c|c} 
    (Hyper-)Parameter & Value\\[0.5ex]  
    \hline\hline 
    Learning rate & $10^{-4}$ \\[0.3ex]  
    Batch size (uncond. generator) & 1500 \\[0.3ex]  
    Batch size (cond. generator) & 1000 \\[0.3ex]  
    Maximal number of gradient steps & 2500
\end{tabular}
\caption{Hyperparameters used for the training.}
\label{table:training_hyperparameters}
\end{center}
\end{table}
\vspace{-.3cm}
\begin{table}[H]
\begin{center}
\begin{tabular}{c|c} 
    (Hyper-)Parameter & Value\\[0.5ex]  
    \hline\hline 
    Number of timesteps $T$ & $10$ \\[0.3ex]  
    Dimension of RS $N$ & 80 \\[0.3ex]  
    Dimension of generator reservoir $D$ & 80 \\[0.3ex] 
    Initial noise dimension $m$ & $5$ \\[0.3ex] 
    Signature truncation level $L$ & \hspace{-1mm}$4$ \\[0.3ex] 
    Distribution $\mu_{A,\xi}$ of random weights for RSig & $\mathcal{N}(0,1)$\\[0.3ex]
    Distribution $\mu_{B,\lambda}$ of random weights in \eqref{eq:generatorSDE} & $\mathcal{N}(0,1)$\\[0.3ex]
    Activation function $\sigma$ & Sigmoid\\[0.3ex]
    Input size of LSTM & 5 \\[0.3ex]
    Number of hidden layers of LSTM & 2 \\[0.3ex]
    Dimension of hidden layers in LSTM & $64$ \\[0.3ex]
    Activation function LSTM & Tanh
\end{tabular}
\caption{Hyperparameters used for the unconditional generative model.}
\label{table:uncond_hyperparameters}
\end{center}
\end{table}
\vspace{-.3mm}
\begin{table}[H]
\begin{center}
\begin{tabular}{c|c} 
    (Hyper-)Parameter & Value\\[0.5ex]  
    \hline\hline 
    Number of past steps $p$ & $5$ \\[0.3ex]  
    Number of future steps $q$ & $10$ \\[0.3ex]  
    Dimension of RS $N$ & 80 \\[0.3ex]  
    Dimension of generator reservoir $D$ & 80 \\[0.3ex] 
    Initial noise dimension $m$ & $15$ \\[0.3ex] 
    Signature truncation level $L$ & \hspace{-1mm}$4$ \\[0.3ex]
    Distribution $\mu_{A,\xi}$ of random weights for RSig & $\mathcal{N}(0,1)$\\[0.3ex]
    Distribution $\mu_{\tilde{B},\tilde{\lambda}}$ of random weights in \eqref{eq:condgeneratorSDE} & $\mathcal{N}(0,1)$\\[0.3ex]
    Activation function $\sigma$ & Sigmoid
\end{tabular}
\caption{Hyperparameters used for the conditional generative model.}
\label{table:cond_hyperparameters}
\end{center}
\end{table}

\section{Evaluation metrics}\label{section:eval_metrics}
In this section, we introduce the main test metrics measuring the performance of the different generator-discriminator pairs in Section \ref{Section3}. Using the previous notation, consider two $d$-dimensional stochastic processes $X=(X_t)_{t=1,\dots,T}$ and $X^\theta=(X^\theta_t)_{t=1,\dots,T}$.

\subsection{Covariance difference}
We introduce the following covariance-distance between $X$ and $X^\theta$ as 
\begin{equation}\label{eq:cov_diff}
    \text{Cov-Dist}(X,X^\theta) := \sqrt{\sum_{j,k=1}^d \sum_{s,t=1}^T \left\vert \text{Cov}(X^{j}_s,X^k_t) - \text{Cov}(X^{\theta,j}_s,X^{\theta,k}_t)\right\vert^2},
\end{equation}
where Cov$(X^j_s,X^k_t)$ denotes the covariance between the $j$-th entry of $X$ at time $s$ and the $k$-th entry of $X$ at time $t$. In practice, we observe samples $(x^{(i)})_{i=1}^M$ of $X$ and $(x^{\theta,(i)})_{i=1}^M$ of $X^\theta$ respectively. for some $M\in\mathbb{N}$. Then the covariance terms in \eqref{eq:cov_diff} are estimated using the estimators
\begin{equation*}
    \widehat{\text{Cov}}(X^j_s,X^k_t) = \frac{1}{M}\sum_{i=1}^M (x^{(i),j}_s -  \bar{x}^{j}_s)(x^{(i),k}_t -  \bar{x}^{k}_t)
\end{equation*}
and
\begin{equation*}
    \widehat{\text{Cov}}(X^{\theta,j}_s,X^{\theta,k}_t) = \frac{1}{M}\sum_{i=1}^M (x^{\theta,(i),j}_s -  \bar{x}^{\theta,j}_s)(x^{\theta,(i),k}_t -  \bar{x}^{\theta,k}_t),
\end{equation*}
where $\bar{x}^{j}_s := \frac{1}{M}\sum_{i=1}^M x^{(i),j}_s$ denotes the sample mean.

\subsection{Autocorrelation difference}
The autocorrelation-distance between $X$ and $X^\theta$ is defined as 
\begin{equation}\label{eq:acf_diff}
    \text{ACF-Dist}(X,X^\theta) := \sqrt{\sum_{j=1}^d \sum_{k=1}^{\floor{T/2}}\left\vert \frac{\rho_{X^j}(k)}{\rho_{X^j}(0)}- \frac{\rho_{X^{\theta,j}}(k)}{\rho_{X^{\theta,j}}(0)}\right\vert^2},
\end{equation}
where $\rho_{X^i}(k)$ denotes the autocovariance of the $i$-th entry of $X$ with lag $k$ and $\rho_{X^{\theta,i}}(k)$ of $X^{\theta,i}$ respectively. Let $(x^{(i)})_{i=1}^M$ be samples of the process $X$ and $(x^{\theta,(i)})_{i=1}^M$ of $X^\theta$ respectively. Then the autocovariances can be estimated by 
\begin{equation*}
    \hat{\rho}_{X^j}(k) = \frac{1}{M(T-k)} \sum_{i=1}^M \sum_{t=1}^{T-k} x^{(i),j}_t x^{(i),j}_{t+k} - \frac{1}{M^2(T-k)^2} \left(\sum_{i=1}^M \sum_{t=1}^{T-k} x^{(i),j}_t\right)\left(\sum_{i=1}^M \sum_{t=1}^{T-k} x^{(i),j}_{t+k}\right)
\end{equation*}
and 
\begin{equation*}
    \hat{\rho}_{X^{\theta,j}}(k) = \frac{1}{M(T-k)} \sum_{i=1}^M \sum_{t=1}^{T-k} x^{\theta,(i),j}_t x^{\theta,(i),j}_{t+k} - \frac{1}{M^2(T-k)^2} \left(\sum_{i=1}^M \sum_{t=1}^{T-k} x^{\theta,(i),j}_t\right)\left(\sum_{i=1}^M \sum_{t=1}^{T-k} x^{\theta,(i),j}_{t+k}\right)
\end{equation*}

Note that for the conditional generator we do not have samples for the future real path, since there is only one trajectory. In this case the autocovariance is estimated with $M=1$.

\section{Kernel density estimate plots for the unconditional generator}
The following plots display kernel density estimates (KDEs) for out-of-sample test data and generated data for paths of an AR$(1)$-process, log-returns of S\&P 500 and FOREX EUR/USD exchange rates. The synthetic data was generated by our proposed unconditional generator model with the generator-discriminator pair \eqref{eq:gen_disc_pair_uncond}. In particular, we considered the dimension of the randomised signature in \eqref{eq:randsigW1} to be $N=80$ and the dimension of the reservoir process $R$ in \eqref{eq:generatorSDE} to be $D=80$. The corresponding model was trained with $2500$ gradient steps.

\begin{figure}[H]
    \centering
    \begin{minipage}{.5\textwidth}
        \centering
        \includegraphics[width=\linewidth]{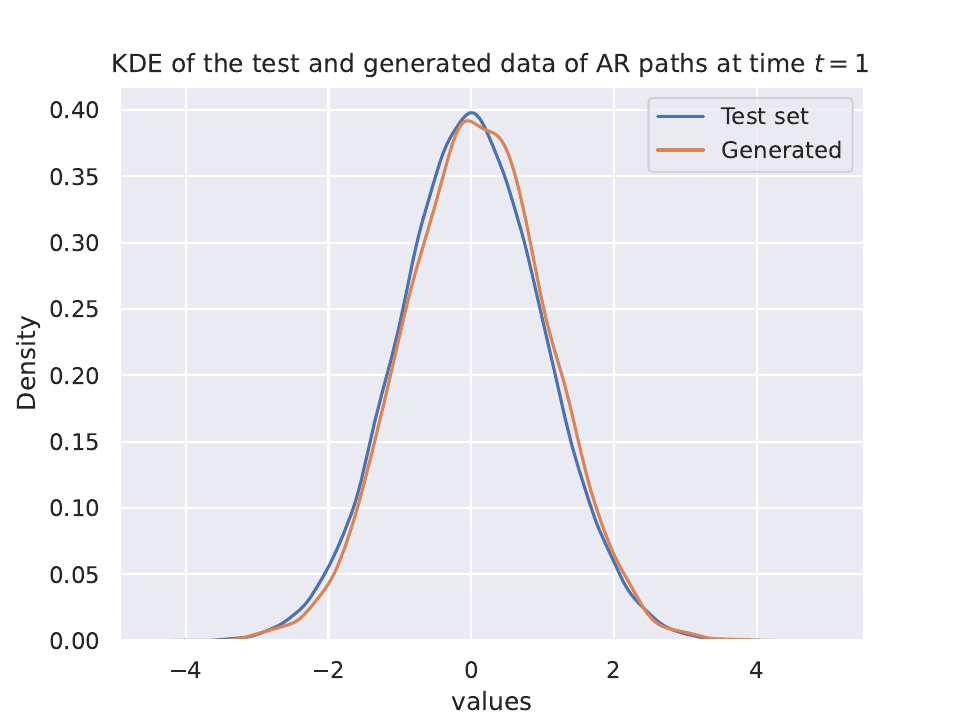}
    \end{minipage}%
    \begin{minipage}{0.5\textwidth}
        \centering
        \includegraphics[width=\linewidth]{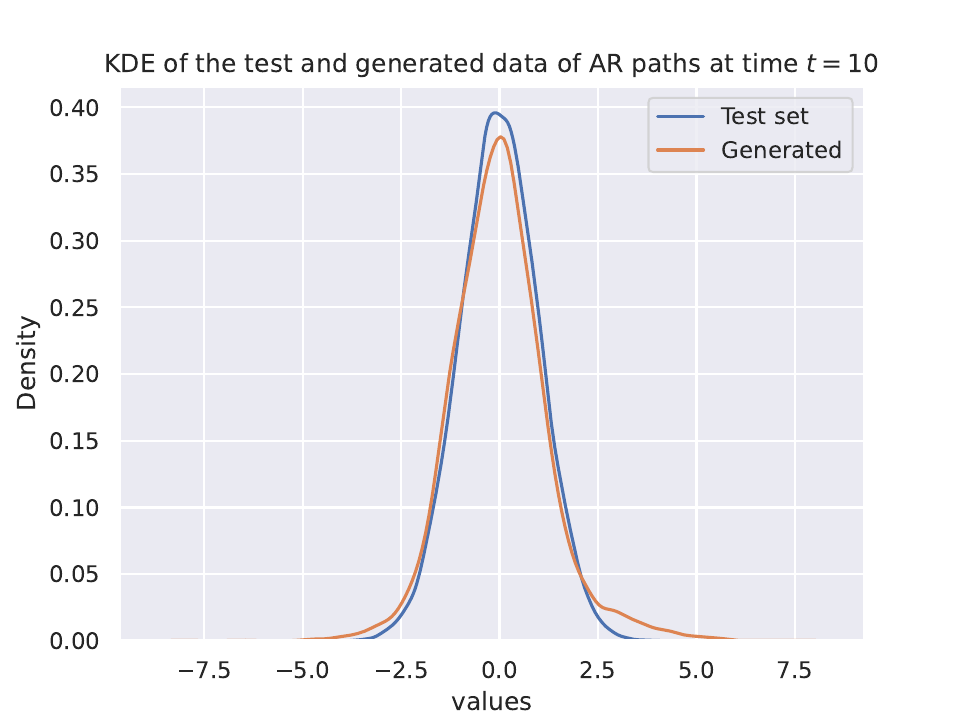}
    \end{minipage}
    \caption{KDEs for generated and test data of values of a AR(1)-process with $\varphi=0.1$ at two different timepoints.}
    \label{fig:kde_ar}
\end{figure}

\begin{figure}[H]
    \centering
    \begin{minipage}{.5\textwidth}
        \centering
        \captionsetup{width=.9\linewidth}
        \includegraphics[width=.9\textwidth]{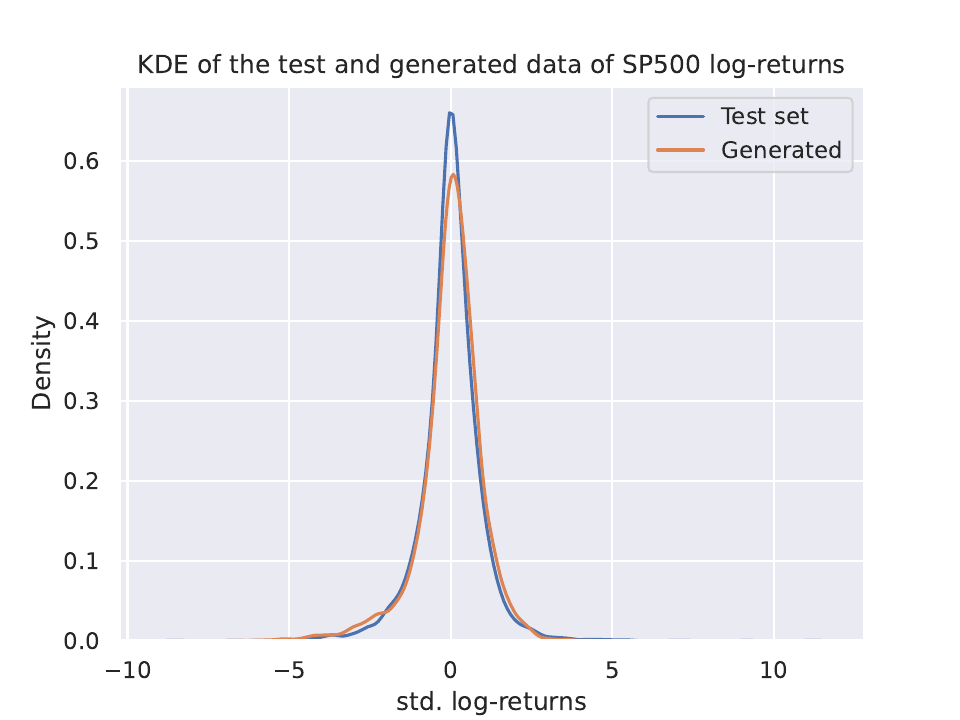}
        \caption{KDEs for generated and test data of S\&P 500 log-returns.}
        \label{fig:kde_sp500}
    \end{minipage}%
    \begin{minipage}{0.5\textwidth}
        \centering
        \captionsetup{width=.9\linewidth}
        \includegraphics[width=.9\textwidth]{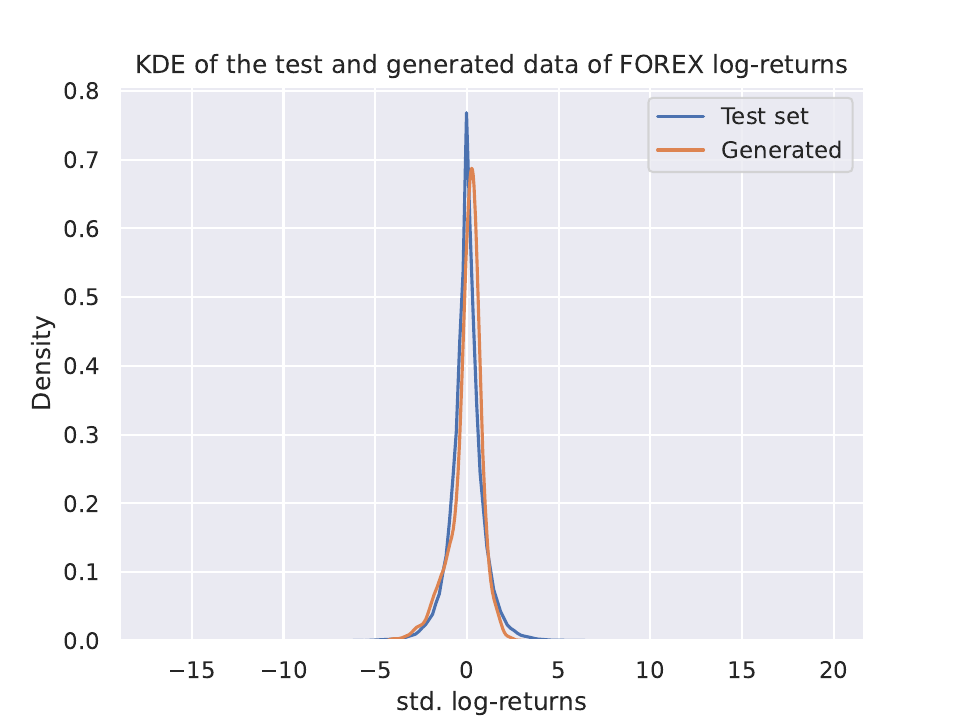}
        \caption{KDEs for generated and test data of FOREX EUR/USD log-returns.}
        \label{fig:kde_forex}
    \end{minipage}
\end{figure}

\section{Pseudocode for the conditional generator}\label{section:pseudocode_cond_gen}
Inspired by \cite{Liao2023}, the following pseudocode describes the implementation of the conditional generator proposed in Section \ref{section4}.
\begin{algorithm}
\caption{Conditional Generator}\label{alg:conditional_gen}
\hspace*{\algorithmicindent} \textbf{Input:} Length of the past path $p$, length of the future path $q$, number of samples $M$, samples $(x_i^{\text{past},p})_{i=1}^M$ of $X^{\text{past},p}_{real}$ and $(x_i^{\text{future},q})_{i=1}^M$ of $X^{\text{future},q}_{real}$, dimension of the randomised signature $N$, sampled random matrices $\tilde{B}_1, \tilde{B}^1_2,\dots, \tilde{B}^w_2 \in\mathbb{R}^{N\times N}, \tilde{\lambda}_1, \tilde{\lambda}^1_2,\dots,\tilde{\lambda}^w_2 \in\mathbb{R}^N$, learning rate $\eta$, number of training epochs $N_{epochs}$, batch size $B$\\
 \hspace*{\algorithmicindent} \textbf{Output:} Optimal parameters for the generator $\hat{X}^{\theta,q}$ 
\begin{algorithmic}[1]
    \State Compute $\Delta \text{RS}_p(x_i^{\text{past},p})_{i=1}^M$ and $\Delta \text{RS}_q(x_i^{\text{future},q})_{i=1}^M$.

    \State Solve the OLS problem 
    $$
        \Delta \text{RS}_q(x_i^{\text{future},q}) = \alpha + \beta \Delta \text{RS}_p(x_i^{\text{past},p}) + \varepsilon_i
    $$  

    for every sample $i=1,\dots,M$ to obtain the optimal $(\hat{\alpha},\hat{\beta})$.

    \State Initialise parameters $\theta$ of the generator $\hat{X}^{\theta}$.
    \For{$i=1,\dots,N_{epochs}$}
        \State Initialise loss $\ell(\theta) \gets 0$.
        \State Sample training batch $\{x_j^{\text{past},p}\}_{j=1,\dots,B}$
        \For{$j=1,\dots, B$}
            \State Generate $M$ samples $(x_i^{\text{future},q})_{i=1,\dots,M}$ of the future path using the current configuration of generator $\hat{X}^{\theta}$ given $x_j^{\text{past},p}$
            \State Compute the MC sum
            $$
                \frac{1}{M}\sum_{i=1}^M \Delta \text{RS}_q(x_i^{\text{future},q}). 
            $$
            \State Update the loss
            $$
                \ell(\theta) \gets \ell(\theta) + \left\Vert \hat{\alpha} + \hat{\beta} x_j^{\text{past},p} - \frac{1}{M}\sum_{i=1}^M \Delta \text{RS}_q(x_i^{\text{future},q})] \right\Vert_2
            $$
            \EndFor
        \State Update the parameters
        $$
            \theta \gets \theta - \eta \frac{\partial \ell(\theta)}{\partial \theta}.
        $$
    \EndFor
    \State \Return {$\theta$} 
\end{algorithmic}
\end{algorithm}

\newpage

\vskip 0.2in
\bibliography{sample}

\end{document}